\documentclass[11pt]{article}
\usepackage{macro}

\newcommand{\rF}{\mathrm{F}}

\begin{document}

\title{\LARGE\bf Implicit Regularization of Gradient Flow on One-Layer Softmax Attention}

\author{
Heejune Sheen
\and
Siyu Chen
\and
Tianhao Wang
\and
Harrison H. Zhou
\and
{\small\textit{Department of Statistics and Data Science, Yale University}} 
\and
{
    \small\texttt{\{heejune.sheen, siyu.chen.sc3226, tianhao.wang, huibin.zhou\}@yale.edu}
}
}

\maketitle

\begin{abstract}
We study gradient flow on the exponential loss for a classification problem with a one-layer softmax attention model, where the key and query weight matrices are trained separately.
Under a separability assumption on the data, we show that when gradient flow achieves the minimal loss value, it further implicitly minimizes the nuclear norm of the product of the key and query weight matrices.
Such implicit regularization can be described by a Support Vector Machine~(SVM) problem with respect to the attention weights.
This finding contrasts with prior results showing that the gradient descent induces an implicit regularization on the Frobenius norm on the product weight matrix when the key and query matrices are combined into a single weight matrix for training.
For diagonal key and query matrices, our analysis builds upon the reparameterization technique and exploits approximate KKT conditions of the SVM associated with the classification data.
Moreover, the results are extended to general weights configurations given proper alignment of the weight matrices' singular spaces with the data features at initialization.  
\end{abstract}

\section{Introduction}
Transformer-based models have become state-of-the-art in various domains of machine learning applications, especially for tasks related to natural language processing and computer vision \citep{vaswani2017attention,brown2020language,dosovitskiy2021an,khan2022transformers,yuan2021tokens}.
The empirical success of these models relies on the attention mechanism \citep{vaswani2017attention}, a crucial component of transformers that provides the ability to evaluate and then exploit the correlation across tokens in the input sequence.
Such ability unlocks the unprecedented performance of recent large language models that have transformed many academic and industrial fields, as well as many aspects of daily lives \citep{brown2020language,achiam2023gpt,anthropic2023claude,team2023gemini}.
Yet, our understanding of the training of such attention-based models remains limited.

A number of recent studies have initiated the effort to develop a theoretical understanding of the attention mechanism in transformers.
One of the key questions here is to investigate the training process of attention-based models and thus to understand the properties of the trained models.
In particular, for the setting of classification tasks, there are recent results characterizing different aspects of the training dynamics of attention-based models \citep{Li2023HowDT,Jelassi2022VisionTP,Li2023ATU,Tian2023ScanAS,Oymak2023OnTR,Tarzanagh2023MaxMarginTS,tarzanagh_transformers_2023,Deora2023OnTO}.
Interestingly, it was shown by \cite{tarzanagh_transformers_2023,Tarzanagh2023MaxMarginTS} that the attention mechanism has an intriguing connection with a certain support vector machine (SVM) problem.
This connection provides a new perspective on the attention mechanism and has been used to characterize the implicit regularization of gradient descent on the classification loss.
Specifically, for training a one-layer softmax attention model with a fixed linear prediction head, \cite{tarzanagh_transformers_2023} showed that gradient descent implicitly minimizes the Frobenious norm of the attention weights and meanwhile, it maximizes the margin between the attention scores of the optimal tokens and non-optimal ones, which can be described by an associated SVM problem for separating the tokens.

This type of implicit regularization of gradient descent is closely related to a previous line of work on margin-maximization of gradient-based optimization algorithms for neural networks in classification tasks \citep{Soudry2017TheIB,Gunasekar2018ImplicitBO,Gunasekar2018CharacterizingIB,nacson2019lexicographic, nacson2019convergence, Lyu2019GradientDM, ji2018gradient, Ji2020DirectionalCA, Ji2019CharacterizingTI, Moroshko2020ImplicitBI}.
In the context of neural networks without the attention module, the margin measures the separation between different classes by the decision boundary. However, in the case of classification with an attention model, the margin refers to the extent to which the attention mechanism \emph{separates the optimal tokens in the input sequence from the non-optimal ones}.
Notably, the convergence analysis in \citet{tarzanagh_transformers_2023} applies only to the case where the key and query matrices are combined into a single weight matrix, and it remains open how to track the dynamics when the key and query matrices are trained separately.
In the latter case, by analyzing the regularization path associated with the empirical loss, \citet{tarzanagh_transformers_2023} suggested that the gradient descent would instead implicitly minimize the nuclear norm of the combined attention weights.
In this paper, we confirm this by providing a direct analysis on gradient flow for training the attention model with separate key and query matrices.

We focus on a one-layer softmax attention model formulated as follows: Let $X\in\R^{L\times d}$ and $z\in\R^d$ be the input sequence and query respectively, then the output of the model is
\begin{align}
   \Phi(K, Q; X, z) = v^\top X^\top \texttt{Softmax}(X KQ^\top z),
\end{align}
where $K, Q\in\R^{d\times d_e}$ are the key and query weight matrices and $v \in \mathbb{R}^{d}$ is a fixed linear prediction head.
For such a model, our main result is summarized in the following informal theorem.

\begin{theorem}[Informal]
For binary classification with a one-layer softmax attention model and the exponential loss, 
gradient flow on $(K,Q)$ at convergence perfectly separates the optimal token from non-optimal ones for each input sequence. 
Meanwhile, it implicitly minimizes the nuclear norm of the combined attention weights $W=KQ^\top$.
\end{theorem}

\paragraph{Main contributions.} Specifically, our contributions are as follows:
\begin{enumerate}
   \item For the exponential loss for binary classification, we analyze the gradient flow for training a one-layer softmax attention model with separate key and query matrices.
   Under proper conditions on the initialization and the data, we show that the empirical loss converges to the infimum loss value as the parameter norm of the attention model grows to infinity.
   Assuming the existence of a certain margin-normalized direction of the combined attention weights $W=KQ^\top$, we then investigate the implicit regularization of the gradient flow.
   
   \item 
   Restricting the key and query matrices to be diagonal, we show that the implicit regularization of gradient flow can be described by a nuclear norm minimizing SVM problem \eqref{eq:WDNSVM} for separating the optimal tokens from the non-optimal ones.
   Our result confirms that when $K$ and $Q$ are trained separately, gradient flow implicitly minimizes the nuclear norm of the combined attention weights, in contrast to the Frobenious norm regularization when the key and query matrices are combined into a single weight matrix.
   Technically, we exploit the reparametrization technique to analyze the intertwined dynamics of $K$ and $Q$, and prove the optimality by verifying the approximate KKT conditions of the associated SVM problem.

   \item We further extend the result to more general weights configurations for the key and query matrices $(K, Q)$.
      We show with additional conditions on data that the gradient flow is equivalent to the diagonal case up to certain rotation transformations, and the convergence point is characterized by \eqref{eq:WNSVM}. 
\end{enumerate}


\section{Related work}

\paragraph{Training dynamics of attention-based models} 
Recent works have studied the training dynamics of attention-based models from various perspectives \citep{Li2023HowDT, Jelassi2022VisionTP, Li2023ATU, Tian2023ScanAS, Oymak2023OnTR, Tarzanagh2023MaxMarginTS, tarzanagh_transformers_2023}.
Specifically, it was shown in \cite{Li2023HowDT} that certain topic structure can be learned by a single-layer transformer trained by gradient descent. 
\cite{Jelassi2022VisionTP} showed that vision transformers (ViT) with softmax attention can learn the spatial structure of the data via gradient descent, and \cite{Li2023ATU} studied the generalization of ViT.
The training dynamics of a single-layer transformer for the next token prediction task was studied in \cite{Tian2023ScanAS}.
The training dynamics of prompt-tuning for attention-based models was investigated in \cite{Oymak2023OnTR}.
See also \citet{zhang_trained_2023, huang_-context_2023,nichani2024transformers,chen2024training,huang2024transformers,thrampoulidis2024implicit} for recent studies on the training dynamics of transformers for in-context learning and beyond.

The most relevant works to ours are \cite{Tarzanagh2023MaxMarginTS,tarzanagh_transformers_2023,vasudeva2024implicit}.
The intriguing connections between the attention mechanism and the support vector machine (SVM) were initially derived in \cite{Tarzanagh2023MaxMarginTS,tarzanagh_transformers_2023}, where the authors characterized the implicit regularization of gradient descent via an associated SVM problem.
\cite{vasudeva2024implicit} extended their results by showing the global directional convergence, as well as the convergence rate, of gradient descent under certain data assumptions. 
Notably, these works considered the case where the key and query matrices are combined into a single weight matrix, while we focus on training the key and query matrices separately.
Though not changing the expressiveness of the attention model, the difference in the architecture indeed leads to significantly different training dynamics and the corresponding implicit regularization.
We also remark that the separate training of the key and query matrices is a common practice \citep{vaswani2017attention}.

\paragraph{Implicit regularization}
There have been many recent studies on the implicit regularization of optimization for neural networks for both regression and classification tasks. 
For the regression settings, it has been shown that gradient descent implicitly minimizes certain parameter norms even without any explicit regularization in the empirical loss \citep{Gunasekar2018CharacterizingIB,azizan2018stochastic,gidel2019implicit,woodworth2020kernel,yun2020unifying,li2021happens,azulay2021implicit,zhou2023implicit,haochen2021shape,fan2023understanding}.
For classification tasks, the margin-maximization phenomenon has been observed and studied for gradient descent on exponential-tailed loss 
\cite{Soudry2017TheIB,Gunasekar2018ImplicitBO,Gunasekar2018CharacterizingIB,nacson2019lexicographic, nacson2019convergence, Lyu2019GradientDM, ji2018gradient, Ji2020DirectionalCA,  Ji2019CharacterizingTI, ji2021fast, yun2020unifying, Moroshko2020ImplicitBI}. 
These results have also been extended to other optimization algorithms \citep{Gunasekar2018ImplicitBO, ji2021fast, Ji2019CharacterizingTI}, different training objectives \citep{nacson2019convergence, Ji2019CharacterizingTI}, and homogeneous neural networks \citep{Gunasekar2018ImplicitBO,Lyu2019GradientDM, nacson2019lexicographic, Ji2020DirectionalCA, ji2018gradient,yun2020unifying, Moroshko2020ImplicitBI}.

Compared to these works, the architecture of the attention model is indeed different.
In particular, the softmax operation makes the model non-homogeneous, and consequently, previous analyses and results on homogeneous models do not directly apply to our model.
Another major distinction is that the implicit regularization of the attention model is characterized by SVM separating the tokens, not the labeled data.
This is due to the special structure of the attention model, which is designed to simultaneously process a sequence of tokens.

\paragraph{Matrix factorization and rank minimization}
Finally, we note that our results are also related to the literature on implicit regularization of gradient descent for matrix factorization.
It has been observed that gradient descent exhibits certain implicit low-rank bias in the context of matrix sensing problems \citep{gunasekar2017implicit,li2018algorithmic,arora2019implicit,li2020towards,stoger2021small}.
\cite{gunasekar2017implicit} showed that, with small initialization, gradient flow on matrix factorization implicitly minimizes the nuclear norm, which was further extended to deep matrix factorization models in \cite{arora2019implicit}.
Later, \cite{li2018algorithmic,li2020towards,stoger2021small} showed that gradient descent with small initialization can recover low-rank solutions for matrix factorization.
While our analysis involves the combined attention weights $W=KQ^\top$ as a product of two parameter matrices and yields similar implicit regularization results as in matrix factorization, our problem setting is completely different.

\section{Preliminaries} \label{sec:Prelim}

We first introduce the conventions and notation used throughout the paper.
For any positive integer $n$, we denote $[n] := \{1, \ldots, n\}$.
For a function $x(t)$ on $t$, we denote by $\dot x(t)$ its derivative with respect to $t$.
For a matrix $X$, we denote by $\|X\|_\rF$ and $\|X\|_*$ its Frobenius norm and nuclear norm, respectively.
For a vector $x\in \R^d$, we write $\Diag(x) \in \R^{d\times d}$ as the diagonal matrix with $x$ being its diagonal entries.
For a matrix $X\in\R^{d\times d}$, we define $\diag(X)\in\R^d$ as the vector containing the diagonal entries.
For vectors $x, y \in \R^d$, $x\odot y$ denotes the Hadamard product between $x$ and $y$, i.e., $x\odot y = (x_1y_1, \ldots, x_dy_d)^\top$, and we write $x^{\odot 2} := x \odot x$. 
For a vector $x\in \R^d$, we define the softmax operation as $\S(x) := (e^{x_1}/\sum_{i=1}^de^{x_i}, \ldots, e^{x_d}/\sum_{i=1}^{d} e^{x_i})^\top$.

\subsection{One-layer transformer for binary classification}\label{sec:preliminary_attention}

We consider the binary classification data $\{(X_{i}, z_{i}, y_{i})\}_{i=1}^n$, where for the $i$-th sample, $(X_{i},z_{i}) \in \R^{L \times d} \times \R^d$ is the feature and $y_{i} \in\{ \pm 1\}$ is the label. 
To solve the classification task, we consider a one-layer softmax attention model with a fixed linear prediction head $v\in\R^d$.
Specifically, let a key matrix $K \in \mathbb{R}^{d \times d_e}$ and a query matrix $Q \in \mathbb{R}^{d \times d_e}$ be the trainable model parameters, where $d_e$ is the embedding dimension.
Denote $\theta=(K,Q)$ for simplicity.
Then, given an input sequence $X\in\R^{L\times d}$ and a query vector $z\in\R^d$, the output of the model is
\begin{align}
    \Phi(\theta; X, z) = v^\top X^\top \S(X KQ^\top z), \label{eq:attn_model}
\end{align}
where the softmax function $\S(\cdot)$ is applied row-wise.
This attention model can be viewed as a simplified one-layer transformer, omitting a feed-forward neural network component.
We refer to $W:=KQ^\top$ as the combined attention weights.

For the classification task, each query vector $z_i$ can be interpreted as a feature that is closely related to the significance of the token within the input sequence $X_i$ for determining the label $y_i$. 
We compute the inner product between the query vector $z_i$ and the embedding of each token in the input sequence $X_i$, and then apply the softmax function to obtain the ``importance'' of each token.
Then by taking the weighted sum of the tokens based on the token importance, the model applies the linear prediction head $v$ to make the final prediction.
Following this intuition, we introduce the notion of token score and optimal token.

\begin{definition}[Token score and optimal token, \citealt{Tarzanagh2023MaxMarginTS,tarzanagh_transformers_2023}]
For the binary classification data $\{(X_{i}, z_{i}, y_{i})\}_{i=1}^n$ where each sequence $X_i = [x_{i1}, \ldots, x_{iL}]^\top \in \R^{L \times d}$, the score of each token $x_{il}$ is defined as $\gamma_{il} := y_i v^\top x_{il}$. 
Moreover, for each $i\in[n]$, the index of an \emph{optimal token} is defined to be $\opt(i) \in \argmax_{ l \in [L]} \gamma_{il}$.
\end{definition}

With a fixed linear prediction head $v$, the token score reflects the importance of each token for predicting the label, with the optimal token having the highest score. 
As such, the performance of the attention model is determined by its ability to separate the optimal token from the non-optimal ones.
Throughout the paper, we consider classification data satisfying the following separability assumption \citep{Tarzanagh2023MaxMarginTS, tarzanagh_transformers_2023,vasudeva2024implicit}. 

\begin{assumption}[Separability]
For the binary classification data $\{(X_{i}, z_{i}, y_{i})\}_{i=1}^n$, the optimal token $\opt(i)$ is unique for all $i\in[n]$.
Furthermore, there exists a matrix $W_* \in \R^{d \times d}$ such that $\left(x_{i,\opt(i)}-x_{il}\right)^{\top} W_* z_{i} \geq 1$ for all $l \neq \opt(i)$, $i\in[n]$.
\end{assumption}

For any token $x_{il}$, the quantity $x_{il}^\top W z_i$ represents the corresponding attention score prior to applying the softmax function.
The separability condition ensures that the attention score of the optimal token $x_{i,\opt(i)}$ is strictly larger than those of the non-optimal tokens.
In particular, it has been shown that the separability assumption holds for data with gaussian noise when $d$ is larger than $L$ and $n$ \citep[Theorem 1]{tarzanagh_transformers_2023}.

\subsection{Gradient flow and implicit regularization}\label{sec:preliminary_gf}
We consider the exponential loss $\ell(x) = \exp(-x)$ for the classification task.
Correspondingly, the empirical loss defined over the data $\{(X_{i}, z_{i}, y_{i})\}_{i=1}^n$ is given by
\begin{align}
\mathcal{L}(\theta) = \frac{1}{n}\sum_{i=1}^n\ell(y_i \cdot \Phi(\theta;X_i,z_i)).
\end{align}
We train the model by minimizing the above objective over $\theta$.
In particular, we consider the gradient flow defined as 
\begin{align}
    \dot{\theta}(t) = - \nabla \LC(\theta(t)). \label{eq:gf}
\end{align}
This is the continuous-time version of the commonly used gradient descent algorithm.

Note that the empirical loss $\LC(\theta)$, when it is trained by the gradient descent method, may not converge to the global infimum value due to its non-convexity.
Let $\mathcal{L}^* \triangleq \frac{1}{n}\sum_{i=1}^n \ell(\gamma_{i,\opt(i)})$,
then 
\begin{align*}
    \cL^* = \inf_{\theta} \cL(\theta) = \inf_{W\in\RR^{d\times d}} \frac{1}{n} \sum_{i=1}^n \exp\bigg(-\sum_{l=1}^L \gamma_{il} \sigma(X_i W z_i)_l\bigg),
\end{align*}
where the infimum is achieved when the combined attention weights $W$ separate the optimal tokens from the non-optimal tokens and the norm $\|W\|_\rF$ diverges to infinity.
For any finite $\theta$, it holds that $\LC(\theta) > \LC^*$ since the softmax scores for non-optimal tokens are always positive, i.e., $\S(X_iKQ^\top z_i)_l >0$ for all $l \neq \opt(i)$ by the nature of softmax function.

To facilitate the convergence analysis of gradient flow, we consider the following assumption.
\begin{assumption}[Assumption B.2, \citealt{tarzanagh_transformers_2023}] \label{item:(A1)}

For each $i\in[n]$, there exists a unique optimal token and the non-optimal tokens have the same scores, i.e., $\gamma_{i ,\opt(i)} > \gamma_{i l}$ and $\gamma_{il} = \gamma_{il'}$ for all $l,l' \neq \opt(i)$, where the non-optimal token score is denoted by $\gamma_{i,\nopt}$.
\end{assumption}

We remark that Assumption \ref{item:(A1)} can be relaxed by allowing small perturbations for non-optimal scores, and a slight modification of the current proof suffices to handle the relaxed condition.
We adhere to Assumption \ref{item:(A1)} in this work for ease of presentation.

Under Assumption \ref{item:(A1)}, it can be shown that the parameter norm diverges to infinity and the empirical loss converges as $t \to \infty$.
In such a regime, we are interested in certain directional convergence of the parameters, and the implicit regularization of the gradient flow corresponds to the specific direction to which the parameters converge.
In particular, for the direction of the combined attention weights $W(t)=K(t)Q(t)^\top$, one can relate it to the following SVM problem:
\begin{equation} \label{eq:WNSVM}
    \tag{$\WNSVM$} 
    \min_{\rank(W)\leq d_e} \|W\|_{*} \quad \text{s.t.} \quad (x_{i,\opt(i)}-x_{i l})^{\top} W z_{i} \geq 1, \forall l \neq \opt(i),i \in [n].
\end{equation}
This formulation was first proposed in \cite{tarzanagh_transformers_2023}. 
The problem  \eqref{eq:WNSVM} includes a non-convex rank constraint on $W$, which is trivially satisfied when $d_e \geq d$, rendering \eqref{eq:WNSVM} a convex problem.

The implicit regularization of gradient flow can be characterized by showing that the direction of the combined attention weights $W(t)$ converges to the solution of \eqref{eq:WNSVM}.
This implies that without any explicit regularization in the empirical loss, the combined attention weights are implicitly regularized to have a small nuclear norm among the feasible points of \eqref{eq:WNSVM}.

\section{Main results for diagonal weights}\label{sec:Diagonal}
We first consider restricting the model parameters $K$ and $Q$ to be diagonal matrices and show that gradient flow converges in the direction to the optimal solution of \eqref{eq:WNSVM} under appropriate assumptions on the data.
The results will then be extended to the case of general weights in Section \ref{sec:general matrix}.
In the rest of this section, we present our main results and proof ideas for diagonal weights and then compare the implicit regularization with prior work.

\subsection{Gradient flow implicitly regularizes nuclear norm}

Under the gradient flow \eqref{eq:gf}, the key and query weight matrices $K(t)$ and $Q(t)$ admit the following dynamics:
\begin{align}
\dot{K}(t) = - \nabla_K \mathcal{L}(K(t), Q(t)), \quad
\dot{Q}(t) = - \nabla_Q \mathcal{L}(K(t), Q(t)).
\end{align}
Recall that $W(t) = K(t)Q(t)^\top$.
It can be shown that $\|W(t)\|_\rF$ diverges as $t \to \infty$ under Assumption~\ref{item:(A1)}.
Therefore, to characterize the directional convergence of $W(t)$, we consider a certain normalized version of $W(t)$, which we call the \emph{margin-normalized direction}.

\begin{definition}\label{def:margin_direction}
For the binary classification data $\{(X_i, z_i, y_i)\}_{i=1}^n$ under Assumption \ref{item:(A1)}, the \emph{margin} of any $W\in\RR^{d\times d}$ is defined as
\begin{align}\label{eq:margin}
    \mu(W) := \min_{i\in[n],l\neq \opt(i)}(x_{i,\opt(i)}-x_{il})^\top Wz_i.
\end{align}
Correspondingly, the \emph{margin-normalized direction} of $W$ is defined as
\begin{align}
    \hat W := \frac{W}{\mu(W)}.
\end{align}
\end{definition}

When $\|W(t)\|_\rF$ diverges, we examine the limit of the margin-normalized direction $\hat W(t) = \frac{W(t)}{\mu(t)}$ as $t\to\infty$, under the assumption that the limit exists. 
Note that, if the limit of $\frac{W(t)}{\|W(t)\|_F}$ exists, then the limit of $\hat W(t)$ shares the same direction. 
We use the normalization factor as a margin to demonstrate that the limit of $\hat W(t)$ is an exact global solution to \eqref{eq:WNSVM}.
In existing works, it is often assumed that the $\ell_2$-normalized direction converges and the margin of the limit is strictly positive, which induces the existence of the limit of margin-normalized direction \citep{Gunasekar2018ImplicitBO,nacson2019lexicographic,Moroshko2020ImplicitBI}.
For homogeneous models, a different set of assumptions can be used to show the convergence of the margin-normalized direction \citep{Ji2020DirectionalCA,Lyu2019GradientDM}.

Consider a subclass of the attention model where the key and query weight matrices admit the following form:
\begin{align}
    K = \begin{pmatrix}
        \xi_{k,1} &0 &\ldots &0 &0 &\ldots &0 \\
        0 &\xi_{k,2} &\ldots &0 &0 &\ldots &0\\
        \vdots &\vdots &\ddots &\vdots &\vdots &\ddots &\vdots\\
        0 &0 &\ldots &\xi_{k,d} &0 &\ldots &0
    \end{pmatrix}, \quad Q = \begin{pmatrix}
        \xi_{q,1} &0 &\ldots &0 &0 &\ldots &0 \\
        0 &\xi_{q,2} &\ldots &0 &0 &\ldots &0\\
        \vdots &\vdots &\ddots &\vdots &\vdots &\ddots &\vdots\\
        0 &0 &\ldots &\xi_{q,d} &0 &\ldots &0
    \end{pmatrix},
\end{align}
where $\xi_k, \xi_q \in \R^d$ are trainable parameters.
Here k and q in the subscript stand for key and query, respectively.
Then the combined weight matrix $W=KQ^\top$ is diagonal, and by letting $\beta = \xi_k \odot \xi_q$, we have 
\begin{align}
    W = \begin{pmatrix}
        \beta_{1} &0 &\ldots &0\\
        0 &\beta_{2} &\ldots &0\\
        \vdots &\vdots &\ddots &\vdots\\
        0 &0 &\ldots &\beta_{d}
    \end{pmatrix}.
\end{align}
For the diagonal weights, it suffices to consider the dynamics of only $(\xi_k(t),\xi_q(t))$, and the corresponding margin-normalized direction, $\hat W(t)$, can be characterized by $\hat\beta(t) = \beta(t) / \mu(t)$.
Accordingly, we consider \eqref{eq:WNSVM} with an additional constraint that $W$ is diagonal.
We introduce diagonal SVM:
\begin{equation} \label{eq:WDNSVM}
    \tag{$\WDNSVM$} 
    \min_{\substack{\rank(W)\leq d_e, \\ W \text{ is diagonal} }} \|W\|_{*} \quad \text{s.t.} \quad (x_{i,\opt(i)}-x_{i l})^{\top} W z_{i} \geq 1 \quad \forall l \neq \opt(i),i \in [n].
\end{equation}
When $d_e \geq d$, \eqref{eq:WDNSVM} can be expressed as following vector-form $\ell_1$- SVM: 
\begin{equation} \label{eq:BLSVM}
    \tag{$\ell_1$-SVM} \min _{\beta \in \mathbb{R}^{d}} \|\beta\|_{1} \quad \text{s.t.} \quad  \beta^{\top} (\underbrace{z_i \circ (x_{i,\opt(i)}-x_{il})}_{B_{il}}) \geq 1  \quad \forall \  l \neq \opt(i), \ i \in [n],
\end{equation}
where we denote $B_{i l}:=z_i \circ (x_{i,\opt(i)}-x_{il})$.

To analyze the dynamics of $\beta$, we consider an alternative parametrization $w =\frac{w_{+}^{\odot 2}-w_{-}^{\odot 2}}{2}$ with
parameters $w_{+}, w_{-} \in \R^d$. 
These two parametrizations are equivalent and have been studied as the so-called diagonal linear network in existing literature \citep{woodworth2020kernel,Moroshko2020ImplicitBI,vaskevicius2019implicit,li2021implicit,haochen2021shape,yun2021a,li2022what,li2022implicit,chou2023more,wind2023implicit,even2024s,pesme2024saddle}. 
For example, we can reparametrize $\beta$ with $\xi_k = (w_+ + w_-)/\sqrt 2, \xi_q = (w_+ - w_-)/\sqrt 2$.
More specifically, we have the following lemma on the equivalence of the two parametrizations.

\begin{restatable}{lem}{equivalenceparameterization}
\label{lem:equivalenceparameterization}
    For the loss $\mathcal{L}(\xi_k, \xi_q) = \frac{1}{n}\sum_{i=1}^n\ell(y_iv^\top X^\top \S(X \diag(\beta) z))$, consider the gradient flow on $\xi_k$ and $\xi_q$:
    \begin{align}\label{xi_dynamics}
        \dot{\xi}_k(t) = - \nabla_{\xi_k} \mathcal{L}(\xi_k(t), \xi_q(t)), \quad
         \dot{\xi}_q(t) = - \nabla_{\xi_q} \mathcal{L}(\xi_k(t), \xi_q(t)).
    \end{align}
    Moreover, for the loss $\mathcal{L}(w_{+}, w_{-}) = \frac{1}{n}\sum_{i=1}^n\ell(y_iv^\top X^\top \S(X \diag(w) z))$, consider the gradient flow on $w_{+}$ and $w_{-}$:
    \begin{align} \label{w_dynamics}
        \dot{w}_{+}(t) = - \nabla_{w_{+}} \mathcal{L}(w_{+}(t), w_{-}(t)), \quad
         \dot{w}_{-}(t) = - \nabla_{w_{-}} \mathcal{L}(w_{+}(t), w_{-}(t))
    \end{align}
    Suppose the following holds at the initialization:  
    \begin{align}
         \beta(0)=w(0), \quad \xi_{k}^{\odot 2}(0)+\xi_{q}^{\odot 2}(0) =w_{+}^{\odot 2}(0)+w_{-}^{\odot 2}(0). \label{eq:init_condition_main0}
    \end{align}
    Then, the dynamics of $\beta$ and $w$ are equivalent, i.e., $\beta(t) = w(t)$ for all $t \geq 0$.
\end{restatable}

The alternative parametrization in $(w_+, w_-)$ is useful for analyzing the gradient flow dynamics of $\beta$ as we can formulate the dynamics of $w$ as described in Appendix \ref{sec:beta_dynamics_A}.  
To employ the results of Lemma \ref{lem:equivalenceparameterization}, we consider the initialization $\xi_k(0)$ and $\xi_q(0)$ satisfying the conditions \eqref{eq:init_condition_main0}.
Additionally, we ensure that $\beta(0) = 0$ and that all entries of $\xi_{k}^{\odot 2}(0)+\xi_{q}^{\odot 2}(0)$ are nonzero. 
More precisely, for any vector $\omega_0 = w_+(0) = w_-(0)$, where $\omega_{0,j} \neq 0$ for all $j \in [d]$, 
we align the initialization by letting $\xi_k(0)$ and $\xi_q(0)$ satisfy
\begin{align}
    \xi_k(0) \odot \xi_q(0) = 0, \quad \xi_{k}(0)^{\odot 2}+\xi_{q}(0)^{\odot 2} = 2\omega_0^{\odot2}. \label{eq:init_condition_main}
\end{align}
Setting $\beta(0) = 0$ ensures that there is no directional bias at initialization. 
Moreover, we require all entries of $\xi_{k}^{\odot 2}(0)+\xi_{q}^{\odot 2}(0)$ to be nonzero so that the gradient flow would not be stuck at initialization.

\begin{restatable}{lem}{divergenceGF}
\label{lem:divergenceGF}
Under Assumption \ref{item:(A1)}, suppose the initialization $\xi_k(0)$ and $\xi_q(0)$ satisfy \eqref{eq:init_condition_main} for the dynamics of $\xi_k(t)$ and $\xi_q(t)$ in Lemma~\ref{lem:equivalenceparameterization}.
Then $\lim_{t \to \infty}\|\beta(t)\|_2 = \lim_{t\to\infty} \|\xi_k(t)\odot\xi_q(t)\|_2 = \infty$. 
\end{restatable}
The proof of Lemma \ref{lem:divergenceGF} is given in Appendix \ref{sec:proof_lem2}.
Lemma \ref{lem:divergenceGF} implies that gradient flow does not converge to a finite point.
Instead, we consider the convergence of the margin-normalized direction $\hat\beta(t) = \beta(t) / \mu(\beta(t))$ as $t\to\infty$.
Before we present our results, we introduce another assumption on the loss value at initialization.  
\begin{assumption} \label{item:(A2)} The loss at initialization is upper bounded as follows:
    \begin{align} 
    \mathcal{L}(\theta(0)) \leq \min_{j \in[n]}\left\{ \frac{\exp (-\gamma_{j,\nopt})+\sum_{i\neq j} \exp \left(-\gamma_{i,\opt(i)}\right)}{n} \right\}.
    \end{align}
\end{assumption}

Assumption \ref{item:(A2)}  ensures that the initial loss is small enough. 
A similar assumption was used in \cite{Ji2020DirectionalCA,Lyu2019GradientDM} to analyze the directional convergence of
gradient flow on homogeneous neural networks. 
If we start with $\beta(0) = 0$, Assumption \ref{item:(A2)}  holds when the gap between the optimal token score and the non-optimal token score is sufficiently large.
Now, we are ready to state the main results for the diagonal weights.

\begin{restatable}{thm}{diagonaltheorem} 
   \label{thm:diagonaltheorem}
    Suppose $d_e \geq d$, and consider the gradient flow dynamics of $\xi_k(t) \in \R^d$ and $\xi_q(t)\in \R^d$ in \eqref{xi_dynamics} with initialization satisfying \eqref{eq:init_condition_main}. 
    Under Assumptions \ref{item:(A1)} and \ref{item:(A2)}, suppose the limit $\hat \beta^\infty := \lim_{t\to\infty} \hat\beta(t)$ exists.
    Let $K(t) = \diag(\xi_k(t))$ and $Q(t) = \diag(\xi_q(t))$, correspondingly $W(t) = \diag(\beta(t))$.
    Then $\hat W^\infty = \diag(\hat\beta^\infty)$ is a global solution of \eqref{eq:WDNSVM}.
\end{restatable}

The key message delivered by the above theorem is two-fold:
\begin{enumerate}
    \item \textbf{Global convergence of loss}. 
    Indeed, we first show that under gradient flow, the limit $\hat\beta^\infty$  is a global solution of \eqref{eq:BLSVM}. 
    Note that for the diagonal weights, \eqref{eq:BLSVM} is equivalent to \eqref{eq:WDNSVM} as mentioned earlier. 
    The theorem shows that if the limit of the margin-normalized direction exists, then the limit must be a globally optimal direction such that gradient flow achieves minimal loss as $t\to\infty$. 
    More precisely, the gradient flow converges in the direction of a feasible diagonal matrix $W_f$ of \eqref{eq:WDNSVM} that separates the optimal tokens from the non-optimal tokens for all input sequences. 
    This further implies that $\lim_{t\to\infty} \cL(K(t),Q(t)) = \cL^*.$
    \item \textbf{Implicit regularization} It is shown that gradient flow implicitly regularizes the nuclear norm of combined attention weights. 
    In the case of diagonal weights, this would potentially give rise  to a sparse $\hat \beta^\infty$ as minimizing $\ell_1$-norm induces sparsity. Accordingly, the diagonal matrix $\hat W^\infty = \diag(\hat\beta^\infty)$ is encouraged to be a low-rank matrix.
\end{enumerate}

The detailed proof of Theorem \ref{thm:diagonaltheorem} can be found in Appendix \ref{sec:proof_Thm1}. 
Below we provide the proof sketch to illustrate the main ideas.
\begin{proof}[Proof sketch of Theorem \ref{thm:diagonaltheorem}]

By Lemma \ref{lem:equivalenceparameterization}, we analyze the alternative parametrization $w$ which mirrors the dynamics of the original parametrization $\beta$ under the initialization conditions. 
We introduce a shorthand $\mu(t) \equiv \mu(\beta(t))$.
By taking integral from $0$ to $t$ on \eqref{w_dynamics}, we have 
\begin{align}
& w_{+}(t)=w_0 \circ \exp \left(\frac{1}{n} \int_{0}^{t} \sum_{i=1}^{n} \exp \left(-\gamma_{i}^{\top} \S\left(g_{i}(\tau)\right) \right) \diag(X_i^{\top} \Sigma(g_i(\tau)) \gamma_i z_i^{\top}) \diff \tau \right), \\
& w_{-}(t)=w_0 \circ \exp \left(-\frac{1}{n} \int_{0}^{t} \sum_{i=1}^{n} \exp \left(-\gamma_{i}^{\top} \S\left(g_{i}(\tau)\right) \right) \diag(X_i^{\top} \Sigma(g_i(\tau)) \gamma_i z_i^{\top}) \diff \tau\right),
\end{align}
where  $w_0 := w_{+}(0) = w_{-}(0) \in \R^d$ is an initialization such that $w_{0,j} \neq 0$ for all $j \in [d]$.
Here, we define
$\Sigma\left(g_{i}\right) :=\text { Diag }\left(\S\left(g_{i}\right)\right)-\S\left(g_{i}\right) \S\left(g_{i}\right)^{\top} \in \mathbb{R}^{L \times L}$,
$\gamma_{i}= y_{i} X_i^\top v \in \R^L$, and  $g_i(t) := X_i\diag(\beta(t))z_i \in \R^L$ for $t\geq 0$. 
Then, the dynamics of the margin-normalized direction $\hat \beta(t) = \beta(t)/\mu(t)$ can be expressed as
\begin{align} \label{eq:beta_tilde_w}
    \hat \beta(t)  =\frac{w_0^{\odot 2}}{\mu(t)} \odot \sinh \bigg(- \frac{2}{n} \sum_{i=1}^{n}\sum_{l\neq \opt(i)}  \int_{0}^{t} \exp (-\gamma_{i}^{\top} \S(g_{i}(\tau)))\left(\Sigma\left(g_{i}(\tau)\right) \gamma_{i}\right)_{l} B_{il} \diff \tau\bigg).
\end{align}

Our goal is to show that $\hat \beta^\infty$  satisfies the KKT conditions of \eqref{eq:BLSVM}:
There exist $\lambda_{il} \geq 0$  for all $i \in [n]$ and  $l \neq \opt(i)$, such that 
\begin{enumerate}[(a)]
    \item (Primal feasibility) for all $i\in[n]$ and any $l\neq\opt(i)$, $\beta^\top B_{i l} \geq 1$; \label{item:primal_feasibility}
    \item (Stationarity)  $\sum_{i\in[n], l\neq \opt(i)} \lambda_{i l} B_{i l} \in \partial\|\beta\|_{1}$, where $\partial\|\beta\|_{1}$ denotes the Clarke subdifferential; \label{item:stationarity}
    \item (Complementary slackness) for all $i\in[n]$ and any $l\neq\opt(i)$, $\lambda_{i l}(1-\beta^\top B_{i l})=0$. \label{item:complementary_slackness}
\end{enumerate}
Since $\hat \beta^\infty  = \lim_{t \rightarrow \infty}\hat\beta(t)$, we show that for all sufficiently large $t$, $\hat\beta(t)$ approximately satisfies the KKT conditions.

For the primal feasibility condition \ref{item:primal_feasibility}, we employ \eqref{eq:beta_tilde_w} to derive that the empirical loss decreases over time. Furthermore, we show that the loss converges to the minimal loss value $\cL^*$ under Assumption \ref{item:(A2)}, which then implies the primal feasibility of $\hat\beta(t)$ for sufficiently large $t$.  

For the stationarity condition \ref{item:stationarity}, given that the objective function of \eqref{eq:BLSVM} is a non-smooth function,
we carefully construct $\lambda_{il}(t)\geq 0$ for $\hat\beta(t)$ by analyzing \eqref{eq:beta_tilde_w} as follows.
For all $l \neq \opt(i)$, $i \in [n]$, and $t \geq 0$, we let
\begin{align} \label{eq:lambda_t}
    \lambda_{i l}(t):=\frac{-2}{n \ln \mu(t)} \int_{0}^{t} \exp \left(-\gamma_{i}^\top \S\left(g_{i}(\tau)\right)\right)\left(\Sigma\left(g_{i}(\tau)\right) \gamma_{i}\right)_{l} \diff \tau.
\end{align} 
Then, we show that $\sum_{i\in[n], l\neq \opt(i)} (\lambda_{i l}(t) B_{i l})_j \rightarrow 1$ (or $-1$) when the $j$-th entry of $\hat \beta^\infty$ is positive (or negative), respectively. Moreover, for the entries of $\hat \beta^\infty$ that are zero, the corresponding entries of $\sum_{i\in[n], l\neq \opt(i)}\lambda_{i l}(t) B_{i l}$ are guaranteed to be in $[-1, 1]$ for sufficiently large $t$.

Regarding the complementary slackness condition \ref{item:complementary_slackness}, we show that for all $i\in [n]$, $l \neq \opt(i)$ and sufficiently large $t$, $\lambda_{i l}(t)$ is bounded from above: 
\begin{align}
     \lambda_{i l}(t) \leq C \frac{ \int_{0}^{t} \S_{l}\left(g_{i}(\tau)\right) \diff \tau}{\int_{0}^{t} \sum_{i=1}^{n} \sum_{l \neq \opt(i)} \S_{l}(g_{i}(\tau)) \diff \tau},
\end{align}
where $C >0$ is some constant. Then, we demonstrate that $\lambda_{il}(t)\rightarrow 0$  as $t \rightarrow \infty$ for all $i \in [n],l \neq \opt(i)$ such that $(\hat \beta^\infty)^\top B_{i l}>1$. It can be obtained by comparing the ratio between  $\lambda_{il}(t)$ and $\lambda_{i'l'}(t)$, where indices $i'$ and $l'$ satisfy $(\hat \beta^\infty)^\top B_{i' l'}=1$.

Combining the above arguments, $\hat\beta(t)$ approaches to a point satisfying the KKT conditions as $t$ increases. 
Finally, we apply the KKT approximation point theorem in \cite{Dutta2013ApproximateKP} to get that $\hat \beta^\infty$ satisfies the KKT conditions.  
\end{proof}

\subsection{Comparison with implicit Frobenius norm regularization}

The implicit regularization of gradient descent on $W$ was derived in \cite{tarzanagh_transformers_2023} under Assumption \ref{item:(A1)} with  appropriate initialization, and further studied in \citet{vasudeva2024implicit}.
Specifically, for the loss $\mathcal{L}(W)$ $= \frac{1}{n}\sum_{i=1}^n\ell(y_i\Phi(W;X_i,z_i))$, they showed that gradient descent on $W$ converges in the direction to the solution to the following SVM problem:
\begin{equation}  \label{eq:WFSVM}
    \tag{$\WFSVM$}
    \min_{W \in \mathbb{R}^{d \times d}} \|W\|_{F} \quad \operatorname{s.t.} \quad (x_{i,\opt(i)}-x_{i l})^{\top} W z_{i} \geq 1, \ \forall l \neq \opt(i), \ i \in [n]. 
\end{equation}
Notably, comparing \eqref{eq:WFSVM} and \eqref{eq:WNSVM}, we observe that gradient descent on the combined weight matrix $W$ implicitly minimizes the Frobenius norm of $W$, while gradient flow on $(K,Q)$ implicitly minimizes instead the nuclear norm of $W=KQ^\top$ by Theorem~\ref{thm:diagonaltheorem}.
This suggests that gradient descent exhibits different implicit regularization behavior under different parametrizations of the attention weights.

Similar discrepancy in implicit regularization under different model parametrizations has also been observed in the context of linear model.
Specifically, given a dataset $\{(x_i,y_i)\}_{i=1}^n$, where $x_i \in \R^d$, $y_i \in \R$, the parameter $\beta \in \R^d$ of the linear regression model is learned by minimizing the squared loss $\frac{1}{n}\sum_{i=1}^n (y_i-x_i^\top \beta)^2$. 
Here, each $y_i = x_i^\top\beta^* +\epsilon_i$, where $\beta^* \in \R^d$ is a  true parameter, and $\epsilon_i \overset{\text{iid}}{\sim} \NC(0,1)$ is a Gaussian noise. 
When the model is overparametrized, i.e., $d > n$, it is known that gradient descent on $\beta$ space implicitly minimizes the $\ell_2$ norm of $\beta$ \citep{Gunasekar2018CharacterizingIB}. 
In contrast, if we consider the diagonal linear network where $\beta$ is parametrized as $\beta = w_+^{\odot 2}-w_-^{\odot 2}$, then for a small initialization, it has been shown that gradient descent on $(w_+,w_-)$ induces an implicit regularizer on $\beta$ that behaves like the $\ell_1$ norm \citep{woodworth2020kernel,vaskevicius2019implicit,chou2023more}. 
See also a general framework for understanding the implicit regularization of gradient descent on reparametrized models in \cite{li2022implicit}.

These results indicate how the choice of model parametrization can impact the training process, which further affects the generalization performance of the learned model.
In particular, our results offer insights into the intricate properties of the attention mechanism, where in practice, the key and query matrices are indeed trained separately instead of being combined into a single weight matrix \citep{vaswani2017attention}.
Note that the nuclear norm is a convex relaxation of the rank function, thus the nuclear norm minimization in \eqref{eq:WNSVM} encourages a low-rank structure of the combined attention weights.
Our findings are consistent with the empirical evidence demonstrating that pre-trained transformer weights tend to be low-rank \citep{aghajanyan2020intrinsic,wang2020linformer,choromanski2020rethinking}.

\section{Main results for general weights} \label{sec:general matrix}
We proceed to study the gradient flow for the general case where the key and query matrices $K\in \R^{d\times d_e}$ and $Q\in \R^{d\times d_e}$ are not necessarily diagonal. 
Below we first introduce some additional assumptions on the binary classification data needed for general weights configurations, and then present our main result.

For general weights configurations, the dynamics of individual entries of $K$ and $Q$ exhibit intricate dependencies on other entries and data, making their analysis more challenging.
In this case, we consider the following assumption for the data $\{(X_{i}, z_{i}, y_{i})\}_{i=1}^n$. 

\begin{assumption} \label{item:(B1)} 
    The binary classification data $\{(X_{i}, z_{i}, y_{i})\}_{i=1}^n$ satisfies the following properties:
    \begin{enumerate}[label=(\roman*),ref=(\roman*)]
    \item For all $i \neq j$, $x_{i,\opt(i)}^{\top} x_{j,\opt(j)} = 0$ and $z_{i}^{\top} z_{j} = 0$. \label{assump:orthogonality}
    \item For all $i\in[n]$, $|x_{i,\opt(i)}^\top z_i| = \|x_{i,\opt(i)}\|_2\|z_i\|_2$. \label{assump:parallel}
    \item 
    For each $i\in[n]$, the set of the indices of non-optimal tokens $[L]\setminus \{\opt(i)\}$ can be partitioned into a pairing $\{(x_{i,l_{2k-1}}, x_{i,l_{2k}})\}_{k=1}^{(L-1)/2}$, where for each of the pairs, there exists a constant $c\in [-1, 1)$ such that
    \begin{align*}
        x_{i, \opt(i)}^\top (x_{i,l_{2k-1}} - x_{i,l_{2k}}) =0, \quad
        x_{i,l_{2k-1}} + x_{i,l_{2k}} = 2 c x_{i,\opt(i)}.
    \end{align*}
    \label{assump:pairing}
    \end{enumerate}
\end{assumption}
In the above assumption, condition \ref{assump:orthogonality} states that the optimal tokens of different sequences are orthogonal to each other, and the query tokens are also orthogonal to each other.
Condition \ref{assump:parallel} states that for each sequence, the query token and the optimal token are parallel.
In addition, condition \ref{assump:pairing} implies that for each sequence, the non-optimal tokens can be paired such that their projections onto the optimal token are equal, and their sum is in the direction of the optimal token.

Moreover, by the orthogonality of the optimal tokens and query tokens, we can formulate an alignment property of the key and query matrices $K$ and $Q$, defined as follows.
Under Assumption~\ref{item:(B1)}, there exist two orthogonal matrices $U_K, U_Q\in\RR^{d\times d}$, whose columns are denoted by $\{u^k_1, \ldots, u^k_d\}$ and $\{u^q_1, \ldots, u^q_d\}$ respectively, such that for all $i\in[n]$, 
\begin{align}\label{eq:orthogonal_token}
    \frac{x_{i,\opt(i)}}{\|x_{i,\opt(i)}\|_2} = u_{\pi(i)}^k \in \{u^k_1, \ldots, u^k_d\}, \quad \frac{z_i}{\|z_i\|_2} = u_{\pi(i)}^q \in \{u^q_1, \ldots, u^q_d\},
\end{align}
for some one-to-one mapping $\pi:[n]\to[d]$.

\begin{definition}{(Alignment Property)}\label{def:alignment_main}
Under Assumption~\ref{item:(B1)} on the binary classification dataset $\{(X_{i}, z_{i}, y_{i})\}_{i=1}^n$, let $U_K,U_Q\in\RR^{d\times d}$ be orthogonal matrices such that \eqref{eq:orthogonal_token} holds for all $i\in[n]$.
We say that the key and query weight matrices $K\in\RR^{d\times d_e}$ and $Q\in\RR^{d\times d_e}$ satisfy \emph{the alignment property} with respect to $\{(X_{i}, z_{i}, y_{i})\}_{i=1}^n$ if $K$ and $Q$ can be decomposed into the following structure: 
\begin{align}
    K = U_K \Lambda_K V^\top, \quad Q = U_Q \Lambda_Q V^\top,
\end{align}
where $V\in\RR^{d_e\times d_e}$ is an orthogonal matrix and $\Lambda_K,\Lambda_Q\in\RR^{d\times d_e}$ are diagonal matrices.
\end{definition}
Definition \ref{def:alignment_main} requires that optimal (query) tokens align with the left eigenvectors of $K$ ($Q$), respectively.
When the initialization $K(0)$ and $Q(0)$ satisfy the alignment property with respect to $\{(X_{i}, z_{i}, y_{i})\}_{i=1}^n$, we can show that the alignment is preserved along the whole trajectory, thus significantly simplifying the dynamics of $K(t)$ and $Q(t)$. 
This is formalized in the following lemma.

\begin{restatable}[Gradient flow preserves alignment]{lem}{alignment}\label{lem:alignment}
Under Assumption~\ref{item:(B1)}, suppose the initialization of the gradient flow satisfies the alignment property with respect to $\{(X_{i}, z_{i}, y_{i})\}_{i=1}^n$ in the sense of Definition~\ref{def:alignment_main}, that is, 
\begin{align}
    K(0) = U_K \Lambda_K(0) V^\top, \quad Q(0) = U_Q \Lambda_Q(0) V^\top.
\end{align}
Then for all $t\geq 0$, $K(t)$ and $Q(t)$ given by the gradient flow \eqref{eq:gf} can be decomposed into
\begin{align}
    K(t) = U_K \Lambda_K(t) V^\top, \quad Q(t) = U_Q \Lambda_Q(t) V^\top.
\end{align}
\end{restatable}

Based on the characterization given by Lemma \ref{lem:alignment}, we can analyze the implicit regularization of the gradient flow for general weights. The main result is presented in the following theorem.

\begin{restatable}{thm}{generaltheorem} 
\label{thm:generaltheorem}
Under Assumptions \ref{item:(A1)}, \ref{item:(A2)}, and \ref{item:(B1)}, suppose that the limit $\hat W^\infty \triangleq \lim_{t \rightarrow \infty} \hat W(t)$ exists. 
    Assume that the initialization of the gradient flow satisfies the alignment property with respect to $\{(X_{i}, z_{i}, y_{i})\}_{i=1}^n$ in the sense of Definition~\ref{def:alignment_main}, that is, 
    \begin{align}
        K(0) = U_K \Lambda_K(0) V^\top, \quad Q(0) = U_Q \Lambda_Q(0) V^\top,
    \end{align}
    where $d_e \geq d$.  
    Let $\Lambda_K(t) = \diag(\xi_k(t))$, $\Lambda_Q(t) = \diag(\xi_q(t))$ and $\beta(t) =\xi_k(t)\odot\xi_q(t)$ such that $(\xi_k(0), \xi_q(0))$ satisfy \eqref{eq:init_condition_main}.
    Then, the gradient flow limit  $\hat W^\infty = U_K\diag(\hat\beta^\infty)U_Q^\top$ is a global optimal solution of \eqref{eq:WNSVM}.
\end{restatable}

Theorem \ref{thm:generaltheorem} shows that under appropriate conditions, 
the results of Theorem \ref{thm:diagonaltheorem} can be extended to general matrices $K$ and $Q$.
Similar to the case of diagonal weights, the empirical loss achieved by the gradient flow converges to the global infimum value $\LC^*$ as $t\to\infty$, and the gradient flow also implicitly regularizes the nuclear norm of the combined attention weights.
We conclude this section by providing a proof sketch of Lemma~\ref{lem:alignment} and Theorem \ref{thm:generaltheorem}, and the detailed proofs are deferred to Appendix~\ref{sec:alignment_lem_proof} and Appendix~\ref{sec:proof_Thm2}.

\begin{proof}[Proof sketch of Theorem \ref{thm:generaltheorem}]
We first present the main idea of the proof for the preservation of the alignment property along the gradient flow trajectory (Lemma \ref{lem:alignment}), and then utilize the simplified dynamics to show the optimality of $\widehat W^\infty$.

By direct calculation, we have
\begin{align}
    \dot{K}(t) &= -\nabla_{K} \mathcal{L}(K(t), Q(t))=\frac{1}{n} \sum_{i=1}^{n} \exp (-\gamma_{i}^{\top} \S(g_{i}(t))) X_{i}^{\top} \Sigma(g_{i}(t)) \gamma_{i} z_{i}^{\top} Q(t), \\
    \dot{Q}(t) &=-\nabla_{Q} \mathcal{L}(K(t), Q(t))=\frac{1}{n} \sum_{i=1}^{n} \exp (-\gamma_{i}^{\top} \S(g_{i}(t))) z_{i} \gamma_{i}^{\top} \Sigma(g_{i}(t)) X_{i} K(t),
\end{align}
where $ \Sigma\left(g_{i}\right)=\text { Diag }\left(\S\left(g_{i}\right)\right)-\S\left(g_{i}\right) \S\left(g_{i}\right)^{\top}$ and $g_i(t) = X_iK(t)Q(t)^\top z_i$.
Denote a vector $\psi_i \triangleq X_{i}^{\top} \Sigma\left(g_{i}\right) \gamma_{i} \in \R^{d}$ for each $i\in[n]$.
Observe that the singular vector spaces of $K(t)$ and $Q(t)$ are invariant along the gradient flow trajectory when the alignment property in Definition~\ref{def:alignment_main} is satisfied at initialization, in the sense that ${\psi_i}/{\|\psi_i\|_2} = u^k_{\pi(i)}$ for all $i\in[n]$.
Hence, it suffices to verify this property.

For any $i\in[n]$, $\psi_i$ can be expressed as
\begin{align}
    \psi_i = \sum_{l\in[L], l \neq \opt(i)} (x_{i,\opt(i)}-x_{il}) \S(g_{i})_{\opt(i)}\S\left(g_{i}\right)_l (\gamma_{i,\opt(i)}-\gamma_{i,\nopt}). \label{eq:psi_m}
\end{align}
When $K$ and $Q$ satisfy the alignment property with data $\{(X_{i}, z_{i}, y_{i})\}_{i=1}^n$,
it follows from Assumption \ref{item:(B1)} (iii) that for each $x_{il}$, there exists a unique $l'\in[L]$ such that 
\begin{align}
     x_{il}^\top W z_i \!&=\! x_{il}^\top U_K \Lambda_K \Lambda_Q U_Q^\top z_i \! =\! x_{il}^\top \frac{x_{i,\opt(i)}}{\|x_{i,\opt(i)}\|_2} \text{diag}(\Lambda_K \Lambda_Q)_{\pi(i)} \!=\! x_{il'}^\top \frac{x_{i,\opt(i)}}{\|x_{i,\opt(i)}\|_2} \text{diag}(\Lambda_K \Lambda_Q)_{\pi(i)} \!\!= \! x_{il'}^\top W z_i.   \label{eq:softmax_equal_m}
\end{align}
This implies that the paired $x_{il}$ and $x_{il'}$ have the same attention probability given by the softmax operation, i.e., $\S(g_{i})_l$ $=\S(g_{i})_{l'}$.
Then, again by Assumption \ref{item:(B1)} (iii),  all components of $x_{il}$ and $x_{il'}$ in the directions orthogonal to $x_{i,\opt(i)}$ can be canceled out in \eqref{eq:psi_m}, and only the components in the direction of $x_{i,\opt(i)}$ remain.
Consequently,
\begin{align}
    \frac{\psi_i}{\|\psi_i\|_2} = \frac{x_{i,\opt(i)}}{\|x_{i,\opt(i)}\|_2} = u^k_{\pi(i)}. 
\end{align}
Hence, we can apply Lemma \ref{lem:alignment} to simplify the gradient flow dynamics of $K$ and $Q$ and focus on the dynamics of the spectral domain of $K$ and $Q$.

By employing the analogous argument as the proof of Theorem \ref{thm:diagonaltheorem}, it can be shown that $\hat \beta^\infty$
is a global solution to 
\begin{align} 
     \ \min _{\beta \in \mathbb{R}^{d}} \|\beta\|_{1} \quad \operatorname{ s.t. } \quad  \beta^{\top} \left(\sign{(\la x_{i, \opt(i)},z_i\ra)} \la x_{i, \opt(i)}-x_{il}, z_i \ra e_{\pi(i)} \right) \geq 1  \quad \forall \  l \neq \opt(i), \ i \in [n], \label{eq:BLSVM_A0}
\end{align}
where $\sign(\cdot)$ is a sign function and $e_i\in \R^d$ is a vector with $1$ in the $i$th entry and 0 elsewhere. Utilizing the fact that $\hat \beta^\infty$ satisfies the KKT condition of \eqref{eq:BLSVM_A0}, we verify that $\hat W^\infty = U_K\diag(\hat\beta^\infty)U_Q^\top$ satisfies the KKT conditions of \eqref{eq:WNSVM}. For details, we refer to Appendix \ref{sec:proof_Thm2}.
\end{proof}
\section{Conclusion and future work}
We have studied implicit regularization of the one-layer softmax attention model trained by gradient flow for binary classification tasks.
While previous work primarily focused on the training on a single weight matrix that combines key and query matrices, we have analyzed the training dynamics on the key and query matrices.
The distinction in parameterization has enabled us to characterize the different implicit regularizations, minimizing the nuclear norm of the product of key and query matrices.

For future work, it would be interesting to extend our results to more general settings, and attention models. 
Additionally, investigating how the implicit regularizations influence the model's generalization performance would be an intriguing direction.

\bibliographystyle{ims}
\bibliography{LLM.bib}

\newpage
\begin{appendices}
\section{Appendix}
\localtableofcontents

\subsection{Calculation of the gradients}

Given the data $\{(X_{i}, z_{i}, y_{i})\}_{i=1}^n$, where $X_{i} \in \mathbb{R}^{L \times d}, z_{i} \in \mathbb{R}^{d}, y_{i} \in\{ \pm 1\}$,
we minimize the following exponential loss function:
\begin{align}
\mathcal{L}(K, Q)=\frac{1}{n} \sum_{i=1}^{n} \ell\left(\gamma_{i}^{\top} \S\left(g_{i}\right)\right),
\end{align}
where $\ell(x)=\exp (-x)$, $\gamma_{i}= y_{i} X_i^\top v \in \R^L$, $g_{i}=X_{i} K Q^{\top} z_{i} \in \mathbb{R}^{L}$
with the parameter $\{K, Q\} \in  \mathbb{R}^{d \times d_e} \ (d_e \geq d)$, and a fixed $v \in \mathbb{R}^{d}$.

The gradients of $\mathcal{L}(K, Q)$ with respect to $K$ and $Q$ are given as follows:
\begin{align}
& \nabla_{K} \mathcal{L}(K, Q)=-\frac{1}{n} \sum_{i=1}^{n} \exp \left(-\gamma_{i}^{\top} \S\left(g_{i}\right)\right) X_{i}^{\top} \Sigma\left(g_{i}\right) \gamma_{i} z_{i}^{\top} Q, \\
& \nabla_{Q} \mathcal{L}(K, Q)=-\frac{1}{n} \sum_{i=1}^{n} \exp \left(-\gamma_{i}^{\top} \S\left(g_{i}\right)\right) z_{i} \gamma_{i}^{\top} \Sigma\left(g_{i}\right) X_{i} K,
\end{align}
where $ \Sigma\left(g_{i}\right)=\text { Diag }\left(\S\left(g_{i}\right)\right)-\S\left(g_{i}\right) \S\left(g_{i}\right)^{\top} \in \mathbb{R}^{L \times L}$. 

If we restrict $K, Q \in \mathbb{R}^{d\times d_e}$ to be diagonal as in Section \ref{sec:Diagonal}, 
it suffices to consider the parameters:
\begin{align}
& \xi_{k}=\operatorname{diag}(K) \in \mathbb{R}^{d}, \\
& \xi_{q}=\operatorname{diag}(Q) \in \mathbb{R}^{d}, \\
& \beta=\operatorname{diag}(W)=\xi_{k} \circ \xi_{q} \in \mathbb{R}^{d}.
\end{align}
The gradient vectors evaluated at $\xi_{k}$ and $\xi_{q}$ are
$$
\begin{aligned}
& \nabla_{\xi_k} \mathcal{L}\left( \xi_k, \xi_q\right)=-\frac{1}{n} \sum_{i=1}^{n} \exp \left(-\gamma_{i}^{\top} \S\left(g_{i}\right)\right) \operatorname{diag}\left(X_{i}^{\top} \Sigma\left(g_{i}\right) \gamma_{i} z_{i}^{\top}\right) \circ \xi_{q}, \\
& \nabla_{\xi_q} \mathcal{L}\left( \xi_k, \xi_q\right)=-\frac{1}{n} \sum_{i=1}^{n} \exp \left(-\gamma_{i}^{\top} \S\left(g_{i}\right)\right) \operatorname{diag}\left(X_{i}^{\top} \Sigma\left(g_{i}\right) \gamma_{i} z_{i}^{\top}\right) \circ \xi_{k},
\end{aligned}
$$
where with a slight abuse of notation, we let $g_i = X_i \left( \bn \circ z_{i} \right)$.

\subsection{Proof of Lemma \ref{lem:equivalenceparameterization}}
For the reader's convenience, we restate Lemma \ref{lem:equivalenceparameterization}.
\equivalenceparameterization*

\begin{proof}
Consider the gradient flow dynamics of $\beta$ and $w$.
Let $g_i(\beta (t)) = X_i \left( \beta(t) \circ z_{i} \right)$, 
and $\ell_i(\beta(t)) = \exp(-\gamma_i^{\top} \S(g_i(\beta(t))))$.
For the parametrization $\beta=\xi_{k} \circ \xi_{q}$, the gradient flow dynamics are
\begin{align}
 \dot{\xi}_{k}(t)& =-\nabla_{\xi_k} \mathcal{L}\left( \xi_k(t), \xi_q(t)\right) \\
& = -\frac{1}{n} \sum_{i=1}^{n} \ell^{\prime}_i(\beta(t)) \diag(X_i^{\top} \Sigma(g_i(\beta(t))) \gamma_i z_i^{\top}) \circ \xi_{q}(t), \\
 \dot{\xi}_{q}(t)& =-\nabla_{\xi_q} \mathcal{L}\left( \xi_k(t), \xi_q(t)\right) \\
& = -\frac{1}{n} \sum_{i=1}^{n} \ell^{\prime}_i(\bn(t)) \diag(X_i^{\top} \Sigma(g_i(\bn(t))) \gamma_i z_i^{\top}) \circ \xi_{k}(t), \\
\dot{\beta}(t)& =\dot{\left(\xi_{k}(t) \circ \xi_{q}(t)\right)} \\
& =\dot{\xi}_{k}(t) \circ \xi_{q}(t)+\xi_{k}(t) \circ \dot{\xi}_{q}(t) \\
& =-\frac{1}{n} \sum_{i=1}^{n} \ell_i^{\prime}(\bn(t)) \diag(X_i^{\top} \Sigma(g_i(\bn(t))) \gamma_i z_i^{\top})\circ\left(\xi_{k}^{\odot 2}(t) + \xi_{q}^{\odot 2}(t)\right). \label{eq:beta_nu_dynamic}
\end{align}
\noindent
For the alternative parametrization, $w=\frac{w_{+}^{\odot 2}-w_{-}^{\odot 2}}{2}$, the gradient flow dynamics are expressed as
\begin{align}
 \dot{w}_{+}(t)& =- \nabla_{w_{+}} \mathcal{L}(w_{+}(t), w_{-}(t)) \\
& = -\frac{1}{n} \sum_{i=1}^{n} \ell^{\prime}_i(\bw(t)) \diag(X_i^{\top} \Sigma(g_i(\bw(t))) \gamma_i z_i^{\top}) \circ w_{+}(t), \label{eq:wp_dynamic}\\
 \dot{w}_{-}(t)& =- \nabla_{w_{-}} \mathcal{L}(w_{+}(t), w_{-}(t)) \\
& = \frac{1}{n} \sum_{i=1}^{n} \ell^{\prime}_i(\bw(t)) \diag(X_i^{\top} \Sigma(g_i(\bw(t))) \gamma_i z_i^{\top}) \circ w_{-}(t), \label{eq:wm_dynamic} \\
\dot{w}(t)& =\frac{w_{+}^{\odot 2}(t)-w_{-}^{\odot 2}(t)}{2} =\dot{w}_{+}(t) \circ w_{+}(t) - \dot{w}_{-}(t) \circ w_{-}(t)  \\
& =-\frac{1}{n} \sum_{i=1}^{n} \ell^{\prime}_i(\bw(t)) \diag(X_i^{\top} \Sigma(g_i(\bw(t))) \gamma_i z_i^{\top}) \circ \left(w_{+}^{\odot 2}(t)+w_{-}^{\odot 2}(t)\right). \label{eq:beta_w_dynamic}
\end{align}
\noindent
Moreover, we have the dynamics of $\xi_{k}^{\odot 2}(t)+\xi_{q}^{\odot 2}(t)$ and $w_{+}^{\odot 2}(t)+w_{-}^{\odot 2}(t)$:
\begin{align}
\dot{\left(\xi_{k}^{\odot 2}(t)+\xi_{q}^{\odot 2}(t)\right)}& = 2 \xi_{k}(t) \circ \dot{\xi}_{k}(t)+2 \xi_{q}(t) \circ \dot{\xi}_{q}(t) \\
& =-\frac{4}{n} \sum_{i=1}^{n}  \ell_i^{\prime}(\bn(t)) \diag(X_i^{\top} \Sigma(g_i(\bn(t))) \gamma_i z_i^{\top}) \circ \beta(t),  \label{eq:nu_dynamic}\\
\dot{\left(w_{+}^{\odot 2}(t)+w_{-}^{\odot 2}(t) \right)}& =2 w_{+}(t) \circ \dot{w}_{+}(t)+2 w_{-}(t) \circ \dot{w}_{-}(t) \\
& =-\frac{4}{n} \sum_{i=1}^{n}  \ell_i^{\prime}(\bw(t)) \diag(X_i^{\top} \Sigma(g_i(\bw(t))) \gamma_i z_i^{\top}) \circ w(t). \label{eq:w_dynamic}
\end{align}
Note that, 
if $\xi_{k}^{\odot 2}(t)+\xi_{q}^{\odot 2}(t)=w_{+}^{\odot 2}(t)+w_{-}^{\odot 2}(t)$ and $\beta(t)=w(t)$,
then
$\dot{\beta}(t)=\dot{w}(t)$ follows by \eqref{eq:beta_nu_dynamic} and \eqref{eq:beta_w_dynamic}. 
On the other hand, if $\beta(t)=w(t),$ 
then $ \dot{(\xi_{k}^{\odot 2}(t)+\xi_{q}^{\odot 2}(t))}=\dot{(w_{+}^{\odot 2}(t)+w_{-}^{\odot 2}(t))} $ by  \eqref{eq:nu_dynamic} and \eqref{eq:w_dynamic}.

By taking the integral on both sides of   \eqref{eq:nu_dynamic} and \eqref{eq:w_dynamic} from $0$ to $t$ and combining them  with 
\eqref{eq:beta_nu_dynamic} and \eqref{eq:beta_w_dynamic} respectively, we have

\begin{align}
 \dot{\beta}(t)&= -\frac{1}{n} \sum_{i=1}^{n}  \ell_i^{\prime}(\bn(t)) \diag(X_i^{\top} \Sigma(g_i(\bn(t))) \gamma_i z_i^{\top}) \\
 & \circ\cbr{\int_{0}^{t}-\frac{4}{n} \sum_{i=1}^{n}  \ell_i^{\prime}(\bn(\tau)) \diag(X_i^{\top} \Sigma(g_i(\bn(\tau))) \gamma_i z_i^{\top}) \circ \beta(\tau) \diff \tau + \xi_k^{\odot 2}(0)+ \xi_q^{\odot 2}(0)}. \label{eq:beta_nu_dynamic2}\\
\dot{w}(t) &=-\frac{1}{n} \sum_{i=1}^{n}  \ell_i^{\prime}(\bw(t)) \diag(X_i^{\top} \Sigma(g_i(\bw(t))) \gamma_i z_i^{\top}) \\
& \circ \cbr{\int_{0}^{t}-\frac{4}{n} \sum_{i=1}^{n}  \ell_i^{\prime}(\bw(\tau)) \diag(X_i^{\top} \Sigma(g_i(\bw(\tau))) \gamma_i z_i^{\top}) \circ w(\tau) \diff \tau + w_+^{\odot 2}(0)+ w_-^{\odot 2}(0)}. 
\label{eq:beta_w_dynamic2}
\end{align}

By \eqref{eq:beta_nu_dynamic2} and \eqref{eq:beta_w_dynamic2}, the dynamics of $\beta$ and $w$ are equivalent when initialized as follows:
\begin{align}
\beta(0)=w(0), \quad \xi_{k}^{\odot 2}(0)+\xi_{q}^{\odot 2}(0)=w_{+}^{\odot 2}(0)+w_{-}^{\odot 2}(0). \label{eq:init_condition_A}
\end{align}
\end{proof}
\subsection{Gradient flow dynamics of $\beta$} \label{sec:beta_dynamics_A}
In this section, we present  
the dynamics of $w$ with initialization satisfying \eqref{eq:init_condition_A} which is equivalent to 
the dynamics of $\beta$ by Lemma \ref{lem:equivalenceparameterization}. 
For $\bw(0) =0$, it is equivalent to $w_+(0) = w_-(0)$. 
We use the notation $\beta = w =\frac{w_{+}^{\odot 2}-w_{-}^{\odot 2}}{2}$ for simplicity.

\begin{lemma} \label{lemma:beta_dynamics}
    For all $t \geq 0$, we have  
    \begin{align} \label{eq:beta_A}
        \beta(t)=\omega_0^{\odot 2} \circ \sinh (B \phi(t)),
    \end{align}
    where  $w_0 := w_{+}(0) = w_{-}(0) \in \R^d$ is an initialization such that $w_{0,j} \neq 0$ for all $j \in [d]$, 
    $B \in \R^{d \times n(L-1)}$ is a matrix with columns $B_{il} := (x_{i,\opt(i)}-x_{il})\circ z_i$,
    and $\phi \in \mathbb{R}^{n(L-1)}$ is a vector such that its entries are expressed as
    \begin{align}
        \phi_{i l}(t) = - \frac{2}{n} \int_{0}^{t} \exp \left(-\gamma_{i}^{\top} \S\left(g_{i}(\tau)\right)\right)\left(\Sigma\left(g_{i}(\tau)\right) \gamma_{i}\right)_{l} \diff \tau
    \end{align}
    for all $i \in [n]$ and $l \neq \opt(i)$.
    Also, we have  
    \begin{align}
        \dot{\beta}(t) = \frac{2}{n} \sqrt{\beta^{\odot 2}(t)+ w_0^{\odot 4} } \circ \rbr{\sum_{i=1}^{n} \sum_{l \neq \opt(i)}-\exp \left(-\gamma_{i}^{\top} \S\left(g_{i}(t)\right)\right)\left(\Sigma\left(g_{i}(t)\right) \gamma_{i}\right)_{l} B_{i l}}.
    \end{align}
\end{lemma}
\begin{proof}
For simplicity, we consider the initialization $w_0 = \alpha\mathbf{1},$
where $\alpha \in \R$ and $\mathbf{1} \in \R^d$ denotes all one vector. 
For any initializations that satisfy \eqref{eq:init_condition_A}, the dynamics can be derived similarly. 
Let $g_i(t) = X_i \left( \beta(t) \circ z_{i} \right)$.
By taking the integral on both sides of  \eqref{eq:wp_dynamic} and \eqref{eq:wm_dynamic} from $0$ to $t$, we have
\begin{align}
& w_{+}(t)=w_{+}(0) \circ \exp \left(\frac{1}{n} \int_{0}^{t} \sum_{i=1}^{n} \exp \left(-\gamma_{i}^{\top} \S\left(g_{i}(\tau)\right) \right) \diag(X_i^{\top} \Sigma(g_i(\tau)) \gamma_i z_i^{\top}) \diff \tau \right), \\
& w_{-}(t)=w_{-}(0) \circ \exp \left(-\frac{1}{n} \int_{0}^{t} \sum_{i=1}^{n} \exp \left(-\gamma_{i}^{\top} \S\left(g_{i}(\tau)\right) \right) \diag(X_i^{\top} \Sigma(g_i(\tau)) \gamma_i z_i^{\top}) \diff \tau\right).
\end{align}
With  $w_+(0) = \alpha \mathbf{1}$ and $w_-(0) = \alpha \mathbf{1}$, we derive the dynamics of $\beta(t)$ as follows:
\begin{align}
\beta(t) & =\frac{w_{+}^{\odot 2}(t)-w_{-}^{\odot 2}(t)}{2} \\
& =\alpha^{2} \sinh \left(\frac{2}{n} \int_{0}^{t} \sum_{i=1}^{n} \exp \left(-\gamma_{i}^{\top} \S\left(g_{i}(\tau)\right)\right) \diag(X_i^{\top} \Sigma(g_i(\tau)) \gamma_i z_i^{\top}) \diff \tau\right). \label{eq:beta_w_v0}
\end{align}
Note that 
\begin{align}
\diag(X_i^{\top} \Sigma(g_i) \gamma_i z_i^{\top})  & =z_{i} \circ\left(X_{i}^{\top} \Sigma\left(g_{i}\right) \gamma_{i}\right)  =\sum_{l=1}^{L}\left(z_{i} \circ x_{i l}\right)\left(\Sigma\left(g_{i}\right) \gamma_{i}\right)_{l}, \label{eq:Di}
\end{align}
and
\begin{align}
\left(\Sigma\left(g_{i}\right) \gamma_{i}\right){ }_{l}=\S_{l}\left(g_{i}\right)\left(\gamma_{i l}-\S\left(g_{i}\right)^{\top} \gamma_{i}\right). \label{eq:softmax_der_A}    
\end{align}

Since $\mathbf{1}^{\top} \Sigma\left(g_{i}\right) = 0$, we have $\sum_{l=1}^{L}\left(\Sigma\left(g_{i}\right) \gamma_{i}\right)_{l}=\mathbf{1}^{\top} \Sigma(g_i) \gamma_i=0$. Hence, we have
\begin{align}
\left(\Sigma\left(g_{i}\right) \gamma_{i}\right)_{\opt(i)}=-\sum_{l \neq \opt(i)}\left(\Sigma\left(g_{i}\right) \gamma_{i}\right)_{l}. \label{eq:softmax_der}
\end{align}
By \eqref{eq:Di} and \eqref{eq:softmax_der}, 
we can substitute $x_{i, \opt(i)}$ to each $x_{il}$ in \eqref{eq:Di} and have for \eqref{eq:beta_w_v0} that
\begin{align}
\beta(t) & =\alpha^{2} \sinh \left(- \frac{2}{n}\int_{0}^{t}  \sum_{i=1}^{n} \sum_{l \neq \opt(i)} \exp \left(-\gamma_{i}^{\top}\S\left(g_{i}(\tau)\right)\right)\left(\Sigma\left(g_{i}(\tau)\right) \gamma_{i}\right)_{l}\left(z_{i} \circ x_{i,\opt(i)}-z_{i} \circ x_{i l}\right) \diff \tau\right) \\
& =\alpha^{2} \sinh \left(- \frac{2}{n}\int_{0}^{t} \sum_{i=1}^{n} \sum_{l \neq \opt(i)}^{\infty} \exp \left(-\gamma_{i}^{\top} \S\left(g_{i}(\tau)\right)\right)\left(\Sigma\left(g_{i}(\tau)\right) \gamma_{i}\right)_{l} B_{il} \diff \tau\right). \label{eq:beta} \\
&=  \alpha^{2} \sinh \left(B \phi(t)\right), 
\end{align}
where we recall $B_{i l}:=z_i \circ (x_{i,\opt(i)}-x_{il})$.
The derivative of $\beta(t)$ can be written as  
\begin{align} 
\dot{\beta}(t)=\frac{2\alpha^{2}}{n} \cosh \left(- \frac{2}{n}\int_{0}^{t}  \sum_{i=1}^{n} \sum_{l \neq \opt(i)} \exp \left(-\gamma_{i}^{\top} \S\left(g_{i}(\tau)\right)\right)\left(\Sigma\left(g_{i}(\tau)\right) \gamma_{i}\right)_{l} B_{i l} \diff \tau\right) \\
\circ \rbr{-\sum_{i=1}^{n} \sum_{l \neq \opt(i)}\exp \left(-\gamma_{i}^{\top} \S\left(g_{i}(t)\right)\right)\left(\Sigma\left(g_{i}(t)\right) \gamma_{i}\right)_{l} B_{i l}}, \label{eq:beta_dot}
\end{align}
and by \eqref{eq:beta}, we have
\begin{align} 
\operatorname{arcsinh}\left(\frac{\beta(t)}{\alpha^{2}}\right) = - \frac{2}{n} \int_{0}^{t}  \sum_{i=1}^{n} \sum_{l \neq \opt(i)} \exp \left(-\gamma_{i}^{\top} \S\left(g_{i}(\tau)\right)\right)\left(\Sigma\left(g_{i}(\tau)\right) \gamma_{i}\right)_{l} B_{i l} \diff \tau. \label{eq:arcsinh}
\end{align}
By substituting \eqref{eq:arcsinh} into \eqref{eq:beta_dot}, we obtain
\begin{align}
& \dot{\beta}(t)=\frac{2\alpha^{2}}{n} \cosh \left(\operatorname{arcsinh}\left(\frac{\beta(t)}{\alpha^{2}}\right)\right) \circ \rbr{-\sum_{i=1}^{n} \sum_{l\neq \opt(i)}\exp \left(-\gamma_{i}^{\top} \S\left(g_{i}(t)\right)\right)\left(\Sigma\left(g_{i}(t)\right) \gamma_{i}\right)_{l} B_{i l}} \\
& \quad \quad =\frac{2}{n} \sqrt{\beta^{\odot 2}(t)+\alpha^{4}\mathbf{1} } \circ \rbr{-\sum_{i=1}^{n} \sum_{l \neq \opt(i)}\exp \left(-\gamma_{i}^{\top} \S\left(g_{i}(t)\right)\right)\left(\Sigma\left(g_{i}(t)\right) \gamma_{i}\right)_{l} B_{i l}},
\end{align}
where the last inequality follows from that $\cosh(\operatorname{arcsinh}(x))=\sqrt{x^{2}+1}$.

\end{proof}

\subsection{Proof of Lemma \ref{lem:divergenceGF}} \label{sec:proof_lem2}

\divergenceGF*

\begin{proof}

By Lemma \ref{lem:equivalenceparameterization}, it suffices to consider the dynamics of $\bw(t)$, 
denoted by $\beta(t)$, for simplicity.
The gradient of the loss function $\mathcal{L}(\beta)$ with respect to $\beta$ is given by
\begin{align}
\nabla_{\beta} \mathcal{L}(\beta) =\frac{1}{n} \sum_{i=1}^{n} \sum_{l \neq  \opt(i)}\exp \left(-\gamma_{i}^{\top} \S\left(g_{i}\right)\right)\left(\Sigma\left(g_{i}\right) \gamma_{i}\right)_{l} B_{i l}.
\end{align}
Then, by the chain rule and the description of $\dot{\beta}(t)$ from Lemma \eqref{lemma:beta_dynamics}, we have 
\begin{align}
    \frac{\diff \mathcal{L}}{\diff t}=\left\langle \nabla_\beta \mathcal{L}(\beta(t)), \frac{\diff \beta(t)}{\diff t}\right\rangle  &=-\left\langle\nabla_{\beta} \mathcal{L}(\beta(t)), \frac{2}{n} \sqrt{\beta^{\odot 2}(t)+ w_0^{\odot 4}} \circ \nabla_{\beta} \mathcal{L}(\beta(t))\right\rangle  \\
    &\leq-\frac{2w_{\min}^2}{n}\left\|\nabla_{\beta} \mathcal{L}(\beta(t))\right\|_{2}^{2}, \label{eq:loss_decrease_A}
\end{align}
where $w_{\min} = \min_{i \in [d]}|w_{0,i}|$.
Taking the integral on both sides, we have
\begin{align}
\mathcal{L}(\beta(0))-\lim_{t \rightarrow \infty}\mathcal{L}(\beta(t)) \geq \int_{0}^{\infty} \frac{2w_{\min}^2}{n} \| \nabla_{\beta} \mathcal{L}(\beta(t)) \|_{2}^{2} \diff t. \label{eq:loss_decrease2_A}
\end{align}
Since $\lim_{t \rightarrow \infty}\mathcal{L}(\beta(t))\geq \frac{1}{n}\sum_{i=1}^n \exp(-\gamma_{i,\opt(i)})>0$, the left-hand side of \eqref{eq:loss_decrease2_A} is bounded above by $\mathcal{L}(\beta(0))$ and the integral on the right-hand side goes to a bounded positive constant by the monotone convergence theorem.
It implies that either $\left\|\nabla_{\beta} \mathcal{L}(\beta(t))\right\|_{2}^{2}=0$ for some finite $\beta(t)$ or $\lim_{t \rightarrow \infty}\left\|\nabla_{\beta} \mathcal{L}(\beta(t))\right\|_{2}^{2} = 0$.

First, we show that $\| \nabla_{\beta} \mathcal{L}(\beta(t))\|_2 \neq 0$ for any finite $\beta(t)$.
Under Assumption \ref{item:(A1)}, for any finite time $t$ and all $ l \neq \opt(i), i \in [n]$, it holds that
\begin{align}
\left(\Sigma\left(g_{i}(t)\right) \gamma_{i}\right)_{l} & =  \rbr{\S(g_i(t))\circ \gamma_i - \S(g_i(t))\S(g_i(t))^\top\gamma_i}_l \\
& = -\S(g_i(t))_l(\S(g_i(t))^\top\gamma_i - \gamma_{i,\nopt}) \\
& = -\S(g_i(t))_l\S(g_i(t))_{\opt(i)}(\gamma_{i,\opt(i)}- \gamma_{i,\nopt}) < 0 \label{eq:softmax_der2_A}
\end{align}
Let $\beta_{f}$ be a feasible solution of  \eqref{eq:BLSVM} (i.e, $\beta_{f}^{\top} B_{i l} \geq 1, \quad \forall \ l \neq \opt(i), i\in [n]$). 
By \eqref{eq:softmax_der2_A},  it holds for any finite $\beta(t)$  that  
\begin{align}
    \left\langle\nabla_{\beta} \mathcal{L}(\beta(t)), \beta_{f}\right\rangle &= \left\langle \frac{1}{n} \sum_{i=1}^{n} \sum_{l \neq  \opt(i)}\exp \left(-\gamma_{i}^{\top} \S\left(g_{i}\right)\right)\left(\Sigma\left(g_{i}\right) \gamma_{i}\right)_{l} B_{i l}, \beta_{f}\right\rangle \\
    & \leq \frac{1}{n} \sum_{i=1}^{n} \sum_{l \neq  \opt(i)}\exp \left(-\gamma_{i}^{\top} \S\left(g_{i}\right)\right)\left(\Sigma\left(g_{i}\right) \gamma_{i}\right)_{l} < 0.
\end{align}
It implies that  $\|\nabla_{\beta} \mathcal{L}(\beta(t))\|_2 \neq 0$ for any finite $\beta(t)$.
Hence, it holds that $\lim_{t \rightarrow \infty}\left\|\nabla_{\beta} \mathcal{L}(\beta(t))\right\|_{2}^{2} = 0 $. In addition, we have 
    \begin{align}
        - \norm{\beta_f}_2 \left\|\nabla_{\beta} \mathcal{L}(\beta(t))\right\|_{2} 
        &\le \left\langle\nabla_{\beta} \mathcal{L}(\beta(t)), \beta_{f}\right\rangle \\
        &\le - \frac{1}{n} \sum_{i=1}^n \exp \left(-\gamma_{i}^{\top} \S\left(g_{i}\right)\right) (1- \sigma(g_i)_{\opt(i)}) (\sigma(g_i)_{\opt(i)}) (\gamma_{i, \opt(i)} - \gamma_{i,\nopt}) < 0. 
    \end{align}
    It follows from $\lim_{t \rightarrow \infty} \left\langle\nabla_{\beta} \mathcal{L}(\beta(t)), \beta_{f}\right\rangle =0 $ that $\lim_{t \rightarrow \infty}(1- \sigma(g_i)_{\opt(i)}) \sigma(g_i)_{\opt(i)} =0$ , $\forall i\in[n],$ which cannot be achieved with any finite $\|\beta(t)\|_2$. Therefore, $\lim_{t \rightarrow \infty}\|\beta(t)\|_2 = \infty$.
\end{proof}

\subsection{Proof of Theorem \ref{thm:diagonaltheorem}} \label{sec:proof_Thm1}

We recall Theorem \ref{thm:diagonaltheorem}:
\diagonaltheorem*

Note that \eqref{eq:BLSVM} is a convex problem and the KKT conditions are sufficient for global optimality.
We show that $\hat \bw$ satisfies the KKT conditions of \eqref{eq:BLSVM}. In particular, 
we analyze the dynamics of $\bw(t)$ and exploit the concept of the approximate KKT point,
which is described in the following subsection.

\subsubsection*{Approximate KKT conditions} \label{sec:ApproxKKT}

We introduce the local Lipshitz functions and the Clarke subdifferential.
A function $f:\R^d \to \R$ is locally Lipshitz if for any $x \in \R^d$, 
there exists a neighborhood $U(x)$ of $x$ such that $f\mid_{U(x)}$ is Lipshitz continuous. 
 A locally Lipshitz function is differentiable almost everywhere. The non-differentiable points of locally Lipshitz functions can be characterized by the Clarke subdifferential:
\begin{align}
\partial f(x) \triangleq \operatorname{conv}\{ \lim_{n\to \infty}\nabla f(x_n) \mid \lim_{n \to \infty}x_n = x, \nabla f(x_n)  \text{ exists} \},
\end{align}
where $\operatorname{conv}$ denotes the convex hull.

Consider the optimization problem:
\begin{align} \label{OP}
\min_{\beta \in \R^d} & \ f(\beta)  \\
s.t.  & \ g_i(\beta) \leq 0    \qquad  \forall i \in [m], 
\end{align}
where $f, g_1, \ldots, g_m: \R^d \to \R$ are locally Lipshitz functions.

\begin{definition} [{\upshape KKT point}]
A point $\beta$ is a KKT point of \eqref{OP} if the following conditions are satisfied. 
There exists $\lambda_1, \ldots, \lambda_n \geq 0$ such that 
\begin{align}
& 0 \in \partial f(\beta) + \sum_{i=1}^n \lambda_i \partial g_i(\beta),  \\ 
& g_i(\beta) \leq 0 \quad \forall \ i \in [n], \\
& \lambda_i g_i(\beta) = 0 \quad \forall \ i \in [n].
\end{align}

\end{definition}
Any local minimizer of problem \eqref{OP} is a KKT point when the regularity condition is satisfied.
We consider one of the regularity conditions, Mangasarian-Fromovitz Constraint Qualification (MFCQ), which is defined as follows.

\begin{definition}  [{\upshape Mangasarian–Fromovitz constraint qualification (MFCQ) }]
Consider a feasible point $\beta$ of \eqref{OP}. The MFCQ is satisfied at $\beta$ if there exists $v \in \R^d$ such that for all $i \in [n]$ such that $g_i(\beta) = 0$, 
\begin{align}
  \quad \la s ,  v \ra > 0 \quad \forall  s \in \partial g_i(\beta). 
\end{align}

\end{definition}

Now, we introduce the concept of approximate KKT points and a related theorem in \citep{Dutta2013ApproximateKP}.
\begin{definition} [{\upshape Approximate KKT point}]
For $\eps, \eps', \delta >0$, a point $\beta$ is an $(\eps,\eps', \delta)$-KKT point of \eqref{OP} if there exists $\check{\beta}$ such that $\|\beta - \check{\beta}\|_2 \leq \eps'$ and there exists $\lambda_i \geq 0, I \in \partial f(\check\beta) $, $s_i \in \partial g_i(\check\beta)$ for all $i \in [n]$ such that 
\begin{align}
& \quad \left \| I +  \sum_{i=1}^n \lambda_i s_i(\check\beta) \right \| _2 \leq \eps \\
& \quad  g_i(\beta) \leq 0 \quad \forall  i \in [n], \\
& \quad   \lambda_i g_i(\beta) \geq - \delta \ \quad \forall i \in [n].
\end{align}

\end{definition}

The following theorem is a corollary of Theorem 3.6  in \citep{Dutta2013ApproximateKP}.
It shows that $(\eps,\eps', \delta)$-KKT points can converge to a KKT point. 

\begin{theorem} {\upshape  \citep{Dutta2013ApproximateKP}  } \label{thm:approxKKT_A}
Let  $\{ \eps_k \}$,  $\{ \eps'_k \}$ and $\{ \delta_k \}$ be sequences 
such that $\eps_k \rightarrow 0$, $\eps'_k \rightarrow 0$, and $\delta_k \rightarrow 0.$ Suppose 
that $\beta_k$ is an $(\eps_k, \eps'_k , \delta_k)$-KKT  point of \eqref{OP} for every $k$.
Let MFCQ  hold at $ \beta$.
If $\beta_k \rightarrow \beta$ as $k \rightarrow +\infty$, then $\beta$ is a KKT point of \eqref{OP}. 
\end{theorem}
\subsection*{Applying Theorem \ref{thm:approxKKT_A} to our problem (\ref{eq:BLSVM})}
Now, we prove 
 Theorem \ref{thm:diagonaltheorem}. 
In this proof,
the notation is simplified by denoting $\bw$ as $\beta$.
In essence,  for large enough $t$, we define $\hat\beta(t) \triangleq \frac{\beta(t)}{\mu(t)}$  and demonstrate that $\hat\beta(1), \dots, \hat\beta(k), \dots$ is a sequence of approximate KKT points of \eqref{eq:BLSVM} with some $\{\epsilon_k, \epsilon'_k, \delta_k\}$ converging to zero. Therefore, if the limit $\hat\beta^\infty$ exists, we conclude that $\hat\beta^\infty$ is a KKT point and hence a global solution to \eqref{eq:BLSVM}, satisfying the MFCQ conditions.

\begin{proof}

We show that $\hat \beta(t)$ is an $(\eps(t), \eps'(t), \delta(t))$-KKT point of \eqref{eq:BLSVM} satisfying the following conditions: 
for all $t>t'$, $\exists  \ \check \beta(t)$ such that $\| \hat \beta(t) - \check \beta(t) \|_2 \leq \eps'(t)$, $\ \lambda(t) \in \R^n_{\geq 0}$ and $ I(t) \in \partial \| \check \beta(t) \|_1$
satisfying 

\begin{enumerate}[label=(C\arabic*).,ref=(C\arabic*)]
    \item $ \quad \left \| I(t) -  \sum_{i=1}^n \lambda_{il} B_{il} \right \| _2 \leq \eps(t),$ \label{approxC1_A} 
    \item $ \quad 1 - \hat \beta(t)^\top B_{il} \leq 0 \quad \forall i \in [n],\ l \neq  \opt(i),$  \label{approxC2_A} 
    \item $ \ \  \lambda_{il} (1- \hat \beta(t)^\top B_{il} ) \geq - \delta(t) \quad \forall i \in [n],\ l \neq  \opt(i).$  \label{approxC3_A}
\end{enumerate}
We simplify $\hat \beta(t)$ as follows before we prove the conditions.  
By Lemma \ref{lemma:beta_dynamics}, $\hat \beta(t)$ can be written as
\begin{align}
    \hat\beta(t) & = \frac{\beta(t)}{\mu(t)}= w_0^{\odot 2}\circ \frac{ \sinh (B \phi(t))}{\mu(t)} \\
    & = w_0^{\odot 2}\circ \frac{\left(\exp \left(\ln \mu(t) \frac{B\phi(t)}{\ln \mu(t)} \right) -\exp \left(-\ln \mu(t)  \frac{B\phi(t)}{\ln \mu(t)}\right)\right) }{2\mu(t)} \\
    & = w_0^{\odot 2}\circ \frac{\left(\mu(t) ^{\frac{B\phi(t)}{\ln \mu(t)}}-\mu(t) ^{-\frac{B\phi(t)}{\ln \mu(t)}} \right)}{2 \mu(t)} \\
    & = \frac{w_0^{\odot 2}}{2} \circ \left(\mu(t)^{\frac{B\phi(t)}{\ln \mu(t)}-1} -\mu(t)^{-\frac{B\phi(t)}{\ln \mu(t)}-1} \right) \label{eq:beta_hat_A},
\end{align}
where  $w_0 := w_{+}(0) = w_{-}(0) \in \R^d$ is an initialization such that $\omega_{0,j} \neq 0$ for all $j \in [d]$, 
$B \in \R^{d \times n(L-1)}$ is a matrix with columns $B_{il}$,
and $\phi \in \mathbb{R}^{n(L-1)}$ is a vector with entries 
$\phi_{i l}(t) = -\frac{2}{n}\int_{0}^{t}  \exp \left(-\gamma_{i}^{\top} \S\left(g_{i}(\tau)\right)\right)\left(\Sigma\left(g_{i}(\tau)\right) \gamma_{i}\right)_{l} \diff \tau$.
We show that $\hat \beta(t)$ satisfies conditions \ref{approxC1_A}, \ref{approxC2_A} and \ref{approxC3_A} for large enough $t$.

\subsubsection*{KKT Condition \ref{approxC2_A}.}
Recall that 
\begin{align}
    \nabla_{\beta} \mathcal{L}(\beta(t)) =\frac{1}{n} \sum_{i=1}^{n} \sum_{l \neq   \opt(i)}\exp \left(-\gamma_{i}^{\top} \S\left(g_{i}(t)\right)\right)\left(\Sigma\left(g_{i}(t)\right) \gamma_{i}\right)_{l} B_{i l}.
\end{align}
Notice that $\exp(-\gamma_{i}^{\top} \S\left(g_{i}\right)) \geq \exp(-\gamma_{i, \opt(i)})$ and $\lim_{t \to \infty}\left\|\nabla_{\beta} \mathcal{L}(\beta(t))\right\|_{2} = 0$ by Lemma \ref{lem:divergenceGF}.   
Then, for $\forall i, l \neq  \opt(i)$, it holds that $\lim_{t \rightarrow \infty}\left(-\Sigma\left(g_{i}(t)\right) \gamma_{i}\right)_{l} = 0$.
By  \eqref{eq:softmax_der_A}, it implies that either
\begin{align}
    \lim_{t \rightarrow \infty}\S\left(g_{i}(t)\right)_l = 0 \ \ \text { or } \quad \lim_{t \rightarrow \infty}\S\left(g_{i}(t)\right)_{ \opt(i)} = 0 \label{eq:softmax_convergence_A}
\end{align}
should be satisfied for all $ l \neq  \opt(i)$ and $i \in [n]$.

Under Assumptions  \ref{item:(A1)} and \ref{item:(A2)}, for all $i \in[n]$, we show that 
\begin{align}
\lim _{t \rightarrow \infty} \S\left(g_{i}(t)\right)_{ \opt(i)}=1,
\end{align}
which is equivalent to $\lim _{t \rightarrow \infty} \S\left(g_{i}(t)\right)_l=0$ for all $l \neq  \opt(i)$ and $i \in[n]$.

We prove the above by contradiction. Suppose there exists $j \in [n]$ such that $\lim _{t \rightarrow \infty} \S\left(g_{j}(t)\right)_{ \opt(i)} \neq 1$,
which is equivalent to 
\begin{align}
\lim _{t \rightarrow \infty} \S\left(g_{j}(t)\right)_{ \opt(i)}=0 \label{eq:softmax_convergence_A2}
\end{align}
by \eqref{eq:softmax_convergence_A}. 
Consider the loss:
\begin{align}
\lim _{t \rightarrow \infty} \mathcal{L}(\beta(t))&=\lim _{t \rightarrow \infty} \frac{1}{n} \sum_{i=1}^{n} \exp \left(-\gamma_{i}^{\top} \S\left(g_{i}(t)\right)\right) \\
& \stackrel{(i)}{=} \lim _{t \rightarrow \infty} \frac{1}{n} \sum_{i=1}^{n} \exp \left(-\left(\gamma_{i,  \opt(i)}-\gamma_{i,\nopt}\right) \S\left(g_{i}(t)\right)_{ \opt(i) }-\gamma_{i,\nopt} \right) \\
& \geq \frac{\exp (-\gamma_{i,\nopt})+\sum_{i\neq j} \exp \left(-\gamma_{i,  \opt(i)}\right)}{n} \\
& \geq \min_{j \in [n]} \frac{\exp (-\gamma_{j,\nopt})+\sum_{i\neq j} \exp \left(\gamma_{i,  \opt(i)}\right)}{n},
\end{align}
where $(i)$ follows from $\gamma_{il}=\gamma_{i,\nopt}$ and $\sum_{l=1}^{L} \S_{l}\left(g_{i}\right)=1.$
By Assumption \ref{item:(A2)}, we have $\lim _{t \rightarrow \infty} \mathcal{L}(\beta(t)) \geq \mathcal{L}(\beta(0))$. 
However, this contradicts to  \eqref{eq:loss_decrease_A} 
that the loss decreases.

By \eqref{eq:softmax_convergence_A2}, it holds that for any $i$, $l \neq  \opt(i)$,
\begin{align}
\lim _{t \rightarrow \infty} \frac{\S\left(g_{i}(t)\right)_{ \opt(i)}}{\S\left(g_{i}(t)\right)_l}=\lim _{t \rightarrow \infty} \exp \left( \beta(t)^\top\left((x_{i, \opt(i)}-x_{i l}) \circ z_{i} \right)\right) = \infty,
\end{align}
equivalently,
\begin{align}
\lim _{t \rightarrow \infty}\beta(t)^\top \left(\left(x_{i, \opt(i)}-x_{i l}\right) \circ z_{i}\right)  =\lim _{t \rightarrow \infty} \beta(t)^\top B_{il}=\infty .
\end{align}
Since $\lim _{t \rightarrow \infty} B_{il}^{\top} \beta(t)=\infty$ for all $l \neq  \opt(i)$ and $i \in [n]$, we have $\lim _{t \rightarrow \infty} \mu(t)=\infty$ 
and for large enough $t$,

\begin{align}\label{eq:tilde_beta_A}
\frac{ \beta(t)^\top B_{il}}{\mu(t)} \geq \frac{\min _{i, l\neq  \opt(i)} \beta(t)^\top B_{i l}}{\mu(t)}=1 .
\end{align}
It follows from\eqref{eq:tilde_beta_A} that $\hat{\beta}$ is a feasible solution to  \eqref{eq:BLSVM}. 
Moreover, for large enough $t$, $\hat{\beta}(t)$ is a  feasible point of \eqref{eq:BLSVM}. 
\subsubsection*{KKT Condition \ref{approxC3_A}.}

We construct the dual variable of $\hat \beta(t)$ as $\lambda(t) \in \R^n_{\geq 0}$ given by
\begin{align}
\lambda_{i l}(t):=\frac{-2}{n \ln \mu(t)} \int_{0}^{t} \exp \left(-\gamma_{i}^\top \S\left(g_{i}(\tau)\right)\right)\left(\Sigma\left(g_{i}(\tau)\right) \gamma_{i}\right)_{l} \diff \tau = \frac{1}{\ln \mu(t)} \phi_{i l}(t).
\end{align}
By \eqref{eq:softmax_der2_A}, we have $-\left(\Sigma\left(g_{i}\right) \gamma_{i}\right)_{l}>0$ for all $i, l \neq  \opt(i)$. 
Hence $\phi_{il}(t)>0$ and $\lambda_{i l}(t) \geq 0$ for all finite time $t$.
Next, we show that this construction gives us condition \ref{approxC3_A}.
To do so, 
we define the support vector index set for $\hat \beta^\infty$ as 
\begin{align}
\Gamma =\left\{(i, l) \in[n] \times[L]\ |\ l \neq\right.  \opt(i), \ \left.B_{i l}^{\top} \hat{\beta}^\infty=1\right\},
\end{align}
and the non-support vector index set as 
\begin{align}
 \Gamma^{c}=\left\{(i, l) \in[n] \times[L]\ | \ l \neq  \opt(i), \ B_{i l}^{\top} \hat{\beta}^\infty>1\right\}.
\end{align}
To prove  condition \ref{approxC3_A}, we show that $\lambda_{il}(t)$ is bounded above for all $(i, l) \in \Gamma \cup \Gamma^c$, and 
\begin{align}
\operatorname{limsup}_{t \rightarrow \infty} \lambda_{i l}(t)=0 \quad \forall (i, l) \in \Gamma^{c}.
\end{align}
In particular,  for the upper bound  $\lambda_{il}(t)$, we derive a lower bound for $\ln \mu(t)$ and an upper bound for $\phi_{i l}$.
Since $\exp \left(-\gamma_{i}^\top \S\left(g_{i}(t)\right)\right)$ and $\S\left(g_{i}(t)\right)_{ \opt(i)}$ are bounded above by $1$, 
\begin{align} \label{eq:upperbound_lambda_A}
\phi_{i l}(t) & \leq \frac{2\left(\gamma_{i, \opt(i)}-\gamma_{i,\nopt}\right)}{n} \int_{0}^{t} \S\left(g_{i}(\tau)\right)_l \diff \tau.
\end{align}
For the lower bound for $\ln \mu(t)$, we note that the existence of limit point $\hat \beta^\infty$ implies that $\mu(t)$ and $\|\beta(t)\|_{2}$ have the 
same order of $t$ (i.e., $\mu(t) \asymp\|\beta(t)\|_{2}$).
Hence, there exists a large enough $t_0$ such that for $t \geq t_{0}$,  
$$
\ln \mu(t) \geq C_{1} \ln \|\beta(t)\|_{2}
$$
holds for some constant $C_{1}>0$.
It suffices to derive a lower bound for $\ln \|\beta(t)\|_{2}$. 
 $\|\beta(t)\|_2$  can be  bounded  from below as follows:
\begin{align}
\|\beta(t)\|_{2} &= \frac{1}{2}\left\|w_{+}^{\odot 2}(t)-w_{-}^{\odot 2}(t) \right\|_{2} \\
&\geq c \bigg\|\exp \bigg( \frac{2}{n} \int_{0}^{t} \sum_{i=1}^{n} \sum_{l\neq\opt(i)} \ell_i(\tau) (\gamma_{i, \opt(i)}-\gamma_{i,\nopt}) \S(g_{i}(\tau))_{\opt(i)}  \S(g_{i}(\tau))_l B_{i l} \diff\tau ) \\ 
&\qquad\qquad - \exp \bigg( -\frac{2}{n} \int_{0}^{t} \sum_{i=1}^{n}\sum_{l \neq \opt(i)} \ell_i(\tau) (\gamma_{i, \opt(i)}-\gamma_{i,\nopt}) \S(g_{i}(\tau))_{\opt(i)} \S(g_{i}(\tau))_l B_{i l}\diff\tau \bigg)\bigg\|_{2}  \\
&\stackrel{(iii)}{\geq} C \bigg\|\exp \bigg(\bigg| \frac{2}{n} \int_{0}^{t} \sum_{i=1}^{n}\sum_{l \neq \opt(i)} \ell_i(\tau) (\gamma_{i,\opt(i)}-\gamma_{i,\nopt}) \S(g_{i}(\tau))_{\opt(i)}  \S(g_{i}(\tau))_l B_{i l} \diff\tau\bigg| \bigg\|_{2} \\
&=C \bigg(\sum_{j=1}^{d} \exp \bigg(\bigg|\frac{4}{n} \int_{0}^{t} \sum_{i=1}^{n}\sum_{l \neq \opt(i)} \ell_i(\tau) (\gamma_{i, \opt(i)}-\gamma_{i,\nopt}) \S(g_{i}(\tau))_{\opt(i)} \S(g_{i}(\tau))_l B_{i l} \diff\tau\bigg|_{j} \bigg)\bigg)^{\frac{1}{2}},
\end{align}
where $c>0$ and $C>0$ are constants, and $\ell_i(\tau)= \exp \left(-\gamma_{i}^\top \S\left(g_{i}(\tau)\right)\right)$. Inequality $(iii)$ follows from that $\abr{e^{x}-e^{-x}} \geq c e^{|x|}$ for some $c>0$ and sufficiently large $|x|$.
Taking the logarithm on both sides, we have
\begin{align*}
& \ln \|\beta(t)\|_{2} \\
& \geq \frac{1}{2} \ln \left(\sum_{j=1}^{d} \exp \left(\abr{ \frac{4}{n} \int_{0}^{t} \sum_{i=1}^{n} \sum_{l \neq  \opt(i)} \ell_i(\tau)\left(\gamma_{i, \opt(i)}-\gamma_{i,\nopt}\right) \S\left(g_{i}(\tau)\right)_{ \opt(i)}  \S\left(g_{i}(\tau)\right)_l B_{i l}\diff \tau}_{j} \right) \right)+ \ln C \\
& = \frac{1}{2} \ln \left(d \cdot \frac{1}{d} \sum_{j=1}^{d} \exp \left(\abr{ \frac{4}{n} \int_{0}^{t} \sum_{i=1}^{n}\sum_{l \neq  \opt(i)} \ell_i(\tau)\left(\gamma_{i, \opt(i)}-\gamma_{i,\nopt}\right) \S\left(g_{i}(\tau)\right)_{ \opt(i)}  \S\left(g_{i}(\tau)\right)_l B_{i l}\diff \tau}_{j} \right)\right)+ \ln C \\
& \stackrel{(iv)}{\geq} \frac{1}{2 d} \sum_{j=1}^{d} \abr{ \frac{2}{n} \int_{0}^{t} \sum_{i=1}^{n}\sum_{l \neq  \opt(i)} \ell_i(\tau)\left(\gamma_{i, \opt(i)}-\gamma_{i,\nopt}\right) \S\left(g_{i}(\tau)\right)_{ \opt(i)}  \S\left(g_{i}(\tau)\right)_l B_{i l}\diff \tau}_{j}  +\frac{1}{2} \ln d+\ln C \\
& \stackrel{(v)}{\geq} \frac{1}{2 d} \quad\left\|\frac{4}{n} \int_{0}^{t} \sum_{i=1}^{n}\sum_{l \neq  \opt(i)} \ell_i(\tau)\left(\gamma_{i, \opt(i)}-\gamma_{i,\nopt}\right) \S\left(g_{i}(\tau)\right)_{ \opt(i)}  \S\left(g_{i}(\tau)\right)_l B_{i l}\diff \tau \right\|_{2}+\frac{1}{2} \ln C^2d \\
& =\frac{2}{ d} \max _{\|a\|_{2}=1}\left\langle \int_{0}^{t} \nabla_{\beta}\calL(\beta(\tau)) \diff \tau, a\right\rangle+\frac{1}{2} \ln C^2d \\
& \stackrel{(vi)}{\geq} \frac{2}{ d} \max \int_{0}^{t} \left\langle \nabla_{\beta} \calL(\beta(\tau)), \bar\beta_{f}\right\rangle \diff \tau +\frac{1}{2} \ln C^2d \\
& =\frac{2}{nd}  \int_{0}^{t} \sum_{i=1}^{n} \sum_{l \neq  \opt(i)} \ell_i(\tau)\left(\gamma_{i, \opt(i)}-\gamma_{i,\nopt}\right) \S\left(g_{i}(\tau)\right)_{ \opt(i)}  \S\left(g_{i}(\tau)\right)_l m_{i l}\diff \tau +\frac{1}{2} \ln C^2d\\
& \stackrel{(vii)}{\geq}  C_{2}(t) \int_{0}^{t} \sum_{i=1}^{n} \sum_{l \neq  \opt(i)} \S\left(g_{i}(\tau)\right)_l \diff \tau+\frac{1}{2} \ln C^2d. 
\end{align*}
Inequality $(iv)$ follows from the Jensen's inequality, and $(v)$ follows from $\|x\|_{1} \geq \|x\|_{2}$.
For  $(vi)$, we choose $\bar{\beta}_{f}$  such that $\|\bar{\beta}_{f}\|_{2}=1,$ $\bar\beta_f^\top B_{i l} = m_{i l} > 0$ for all $i, l \neq  \opt(i)$.
For $(vii)$, we choose $ C_{2}(t)=\min _{i, l, \tau \in[0,t]}\left\{\frac{2}{nd} \ell_i(\tau)\left(\gamma_{i, \opt(i)}-\gamma_{i,\nopt}\right) \S\left(g_{i}(\tau)\right)_{ \opt(i)} m_{i l} \right\}$.

Since $\ell_{i}(t)$ and $\S\left(g_{i}(t)\right)_{\opt(i)}$ are bounded below by some positive constant for any $t$, we have
$\lim _{t \rightarrow \infty} C_{2}(t) = C_2>0.$
Hence, we have
\begin{align} \label{eq:lowerbound_lambda_A}
\ln \|\beta(t)\|_{2} \geq C \int_{0}^{t} \sum_{i=1}^{n} \sum_{l \neq  \opt(i)} \S\left(g_{i}(\tau)\right)_l \diff \tau
\end{align}
for sufficiently large $t$ and some constant $C>0$.
Combining \eqref{eq:upperbound_lambda_A} and  \eqref{eq:lowerbound_lambda_A}, for all $i\in n$ and $l \neq  \opt(i)$, we have
\begin{align*}
    \lambda_{i l}(t) = \frac{\phi_{i l}(t)}{\ln \mu(t)} \leq C \frac{ \int_{0}^{t} \S\left(g_{i}(\tau)\right)_l \diff \tau}{\int_{0}^{t} \sum_{i=1}^{n} \sum_{l \neq  \opt(i)} \S\left(g_{i}(\tau)\right)_l \diff \tau}
\end{align*}
for sufficiently large $t$ and some constant $C>0$.

Note that $\lim_{t \rightarrow \infty}\int_{0}^{t} $ $
\sum_{i=1}^{n} \sum_{l \neq  \opt(i)} \S\left(g_{i}(\tau)\right)_l \diff \tau$ goes to infinity by
$\lim_{t \rightarrow \infty}\|\beta(t)\|_{2} = \infty$ and \eqref{eq:beta}.
Moreover, it follows from  $\hat \beta(t)^\top B_{il} \to \hat\beta^\infty(t)^\top B_{il}$ that there exists time $t_0$ such that for $t>t_0$, $\hat \beta(t)^\top B_{il} >1$ holds for all $(i, l) \in \Gamma^{c}$. 
 Then, we have $\lim _{t \rightarrow \infty} \frac{\S\left(g_{i'}(t)\right)_{l'}}{\S\left(g_{i}(t)\right)_l}=\infty$
 (i.e., $\S\left(g_{i}(t)\right)_l =o(\S\left(g_{i'}(t)\right)_{l'})$) 
for all $(i, l) \in \Gamma^{c}$ and $(i',l') \in \Gamma$.

It implies that
\begin{align}
\limsup_{t \rightarrow \infty}\frac{ \int_{0}^{t} \S\left(g_{i}(\tau)\right)_l \diff \tau}{\int_{0}^{t} \sum_{i=1}^{n} \sum_{l \neq  \opt(i)} \S\left(g_{i}(\tau)\right)_l \diff \tau} =0 \qquad \forall (i, l) \in \Gamma^{c}.
\end{align} 
Consequently, $\limsup _{t \rightarrow \infty} \lambda_{i l}(t)=0$ for all $(i, l) \in \Gamma^{c}$.

\subsubsection*{KKT Condition \ref{approxC1_A}.}
We prove the condition \ref{approxC1_A} by constructing $\check \beta(t)$ such that it satisfies $\| \hat \beta(t) - \check \beta(t) \|_2 \leq \eps'(t)$ and the condition \ref{approxC1_A}.
Let $\check \beta_j(t) = \hat \beta_j(t)$ for all $j \in \{ j \in [d] \ | \ \hat \beta^\infty_j \neq 0 \}$ and $\check \beta_j(t) = 0$ for $j \in \{ j \in [d] \ | \ \hat \beta^\infty_j = 0 \}$.
Since $\hat \beta(t) \rightarrow \hat \beta^\infty$, we have $\eps'(t) \rightarrow 0$ as $t \rightarrow \infty$.
We prove that $\eps(t)$ converges to 0. 
 As  derived in \eqref{eq:beta_hat_A}, we have 
\begin{align}
    \hat \beta(t) & = \frac{w_0^{\odot 2}}{2} \circ \left(\mu(t)^{\frac{B\phi(t)}{\ln \mu(t)}-1} -\mu(t)^{-\frac{B\phi(t)}{\ln \mu( t)}-1} \right) \\
    & = \frac{w_0^{\odot 2}}{2} \circ \left(\mu(t)^{B\lambda(t)-1} -\mu(t)^{-B\lambda(t)-1} \right).
\end{align}

Consider the index $j \in \{j\in [d] \ |\  \hat \beta^\infty_j = C_j > 0 \}$. Let $a_j(t):= (B\lambda(t))_j$ and $b_j(t) := \mu(t)^{a_j(t)}$. The $j$th component of $\check \beta(t)$ can be written as
\begin{align}
    \check \beta_j(t) = \hat \beta_j(t) = \frac{w_j^2(0)}{2}\left(b_j(t)\mu(t)^{-1}-b_j(t)^{-1}\mu(t)^{-1} \right).
\end{align}
Equivalently,
\begin{align} \label{eq:beta_j_positive}
 b_j(t)\mu(t)^{-1}-b_j(t)^{-1}\mu(t)^{-1}  = \frac{2\hat \beta_j(t)}{w_j^2(0)} := C_j(t),
\end{align}
where $ \lim_{t \rightarrow \infty} C_j(t) = 2C_j/w_j^2(0)$.
If we represent \blue{ \eqref{eq:beta_j_positive} } in terms of $b_j(t)$, 
\begin{align}
    b_j(t) = \frac{C_j(t)\mu(t) + \sqrt{C_j^2(t)\mu^2(t) +4}}{2}.
\end{align}
Taking the logarithm on both sides, we obtain
\begin{align}
    a_j(t) = \frac{\ln(C_j(t)\mu(t) + \sqrt{C_j^2(t)\mu^2(t) +4}) - \ln 2}{\ln \mu(t)}.
\end{align}
Since $\lim_{t \rightarrow \infty} \mu(t) = \infty$, we have $\lim_{t \rightarrow \infty}a_j(t) = 1$.
Note that for sufficiently large $t>t_0$, $C_j(t) > 0$ and $\partial\|\check \beta(t)\|_1 = 1$. 
Then, it holds that 
\begin{align}
      \lim_{t \rightarrow \infty}\eps_j(t) = \lim_{t \rightarrow\infty}\abr{a_j(t)-\partial\|\check \beta(t)\|_1} = \lim_{t \rightarrow \infty}\abr{a_j(t)-1} = 0. 
\end{align}

Similarly, consider the index $j \in \{j\in [d] \ |\  \hat \beta_j^\infty = -C_j < 0 \}$.
We have the following:
\begin{align}
    & b_j(t)\mu(t)^{-1}-b_j(t)^{-1}\mu(t)^{-1} = \frac{2\check \beta_j(t)}{w_j^2(0)} := -C_j(t), \\
    & b_j(t) = \frac{-C_j(t)\mu(t) + \sqrt{C_j^2(t)\mu^2(t) +4}}{2} = \frac{2}{C_j(t)\mu(t) + \sqrt{C_j^2(t)\mu^2(t) +4}},  \\
    & a_j(t) = -\frac{\ln\left(C_j(t)\mu(t) + \sqrt{C_j^2(t)\mu^2(t) +4}\right) - \ln 2}{\ln \mu(t)}.  
\end{align}
Then, it follows that $\lim_{t \rightarrow \infty}a_j(t) = -1$ and $\lim_{t \rightarrow \infty}\eps_j(t) = 0$.

Finally, consider the index $j \in \{j\in [d] \ |\  \hat \beta_j^\infty =  0 \}$. 
Since $\check \beta_j(t) =0$ and 
$ \lim_{t \rightarrow \infty}\hat \beta_j(t) = \lim_{t \rightarrow \infty}\frac{w_j^2(0)}{2}\left(\mu(t)^{(B\lambda(t))_j-1} -\mu(t)^{-(B\lambda(t))_j-1} \right)$ = 0,
it implies that for sufficiently large $t$, $(B\lambda(t))_j \in [-1, 1] = \partial\|\check \beta{(t)}\|_1 $. 
Consequently, we have
\begin{align}
    \lim_{t \rightarrow \infty}\eps_j(t) =   \lim_{t \rightarrow\infty}\abr{a_j(t)-\partial\|\check \beta(t)\|_1} = 0.
\end{align}

\end{proof}

\subsection{Proof of Lemma \ref{lem:alignment}}\label{sec:alignment_lem_proof}
For reader's convenience, we restate Lemma \ref{lem:alignment} here.

\alignment*

\begin{proof}

For ease of presentation,
 we consider the case $d = d_e$. The same analysis can be applied to the case $d \geq d_e$.
The gradient flow dynamics of $K$ and $Q$ are given by
\begin{align}
& \dot{K}(t) = -\nabla_{K} \mathcal{L}(K(t), Q(t))=\frac{1}{n} \sum_{i=1}^{n} \exp \left(-\gamma_{i}^{\top} \S\left(g_{i}(t)\right)\right) X_{i}^{\top} \Sigma\left(g_{i}(t)\right) \gamma_{i} z_{i}^{\top} Q(t), \\
& \dot{Q}(t) =-\nabla_{Q} \mathcal{L}(K(t), Q(t))=\frac{1}{n} \sum_{i=1}^{n} \exp \left(-\gamma_{i}^{\top} \S\left(g_{i}(t)\right)\right) z_{i} \gamma_{i}^{\top} \Sigma\left(g_{i}(t)\right) X_{i} K(t).
\end{align}

We denote a vector $\psi_i \triangleq X_{i}^{\top} \Sigma\left(g_{i}\right) \gamma_{i} \in \R^{d}$. Then, it can be expressed as
\begin{align}
    \psi_i & = \sum_{l =1}^L x_{il} (\Sigma\left(g_{i}\right) \gamma_{i})_{l} \\
    & \stackrel{(a)}{=} \sum_{l \neq  \opt(i)} (x_{il}-x_{i, \opt(i)}) (\Sigma\left(g_{i}\right) \gamma_{i})_{l} \\
    & \stackrel{(b)}{=} \sum_{l \neq  \opt(i)} (x_{i, \opt(i)}-x_{il}) \S(g_{i})_{ \opt(i)}\S\left(g_{i}\right)_l (\gamma_{i, \opt(i)}-\gamma_{i,\nopt}), \label{eq:psi}
\end{align}
where $(a)$ follows from  \eqref{eq:softmax_der}. Equality $(b)$ follows from 
\eqref{eq:softmax_der2_A}.
Suppose that $K$ and $Q$ satisfy the alignment property with data $\{(X_{i}, z_{i}, y_{i})\}_{i=1}^n$.
By  Assumption \ref{item:(B1)} (iii),  for each $x_{il}$, there exists a $x_{il'}$ such that 
\begin{align}
    x_{il}^\top W z_i  &= x_{il}^\top U_K \Lambda_K \Lambda_Q U_Q^\top z_i \\
    & = x_{il}^\top U_K \Lambda_K \Lambda_Q e_{\pi(i)} \\
    & = x_{il}^\top u^k_{\pi(i)} \text{diag}(\Lambda_K \Lambda_Q)_{\pi(i)} \\
    & = x_{il}^\top \frac{x_{i, \opt(i)}}{\|x_{i, \opt(i)}\|_2} \text{diag}(\Lambda_K \Lambda_Q)_{\pi(i)} \\
    & = x_{il'}^\top \frac{x_{i, \opt(i)}}{\|x_{i, \opt(i)}\|_2} \text{diag}(\Lambda_K \Lambda_Q)_{\pi(i)} \\
    &= x_{il'}^\top W z_i.  
\end{align}
Consequently, for paired $x_{il}$ and $x_{il'}$, we have
\begin{align}
    \S\left(g_{i}\right)_l= \S\left(g_{i}\right)_{l'}. \label{eq:softmax_equal}
\end{align}
Together with \eqref{eq:psi}, \eqref{eq:softmax_equal} and  Assumption \ref{item:(B1)} (iii), it follows that for all $i$,
\begin{align}
    \frac{\psi_i}{\|\psi_i\|_2} = \frac{x_{i, \opt(i)}}{\|x_{i, \opt(i)}\|_2} = u^k_{\pi(i)}. \label{eq:psi_direction}
\end{align}

Observe the dynamics of $K$ and $Q$ as follows:
\begin{align}
     \dot{K} & =\frac{1}{n} \sum_{i=1}^{n} \exp \left(-\gamma_{i}^{\top} \S\left(g_{i}\right)\right) \psi_i z_{i}^{\top} Q \\
     & = \frac{1}{n} \sum_{i=1}^{n} \exp \left(-\gamma_{i}^{\top} \S\left(g_{i}\right)\right) \psi_i z_{i}^{\top} U_Q \Lambda_Q V^\top \\
     & = \frac{1}{n} \sum_{i=1}^{n} \exp \left(-\gamma_{i}^{\top} \S\left(g_{i}\right)\right) u^k_{\pi(i)} e_{\pi(i)}^\top \|\psi_i\|_2 \|z_i\|_2 \Lambda_Q V^\top \\
     & = \frac{1}{n} \sum_{i=1}^{n} \exp \left(-\gamma_{i}^{\top} \S\left(g_{i}\right)\right)  \|\psi_i\|_2 \|z_i\|_2 \text{diag}(\Lambda_Q)_{\pi(i)}u^k_{\pi(i)} {v_{\pi(i)}}^\top ,
\end{align}
where $e_i \in \R^d$ is a vector with $1$ at the $i$th component and $0$ elsewhere. 
\begin{align}
    \dot{Q} & = \frac{1}{n} \sum_{i=1}^{n} \exp \left(-\gamma_{i}^{\top} \S\left(g_{i}\right)\right) z_i \psi_{i}^{\top} K \\
    & = \frac{1}{n} \sum_{i=1}^{n} \exp \left(-\gamma_{i}^{\top} \S\left(g_{i}\right)\right) z_i \psi_{i}^{\top} U_K \Lambda_K V^\top  \\
    & = \frac{1}{n} \sum_{i=1}^{n} \exp \left(-\gamma_{i}^{\top} \S\left(g_{i}\right)\right) u^q_{\pi(i)} e_{\pi(i)}^\top \|\psi_i\|_2 \|z_i\|_2 \Lambda_K V^\top \\
    & = \frac{1}{n} \sum_{i=1}^{n} \exp \left(-\gamma_{i}^{\top} \S\left(g_{i}\right)\right)  \|\psi_i\|_2 \|z_i\|_2 \text{diag}(\Lambda_K)_{\pi(i)}u^q_{\pi(i)} {v_{\pi(i)}}^\top.
\end{align}
The dynamics of $K$ and $Q$ do not alter the column spaces of $U_K$, $U_Q$ and $V$, respectively when $K$ and $Q$ satisfy the alignment property with data $\{(X_{i}, z_{i}, y_{i})\}_{i=1}^n$.
In particular, consider the initialization $K(0) = U_K(0) \Lambda_K V(0)^\top$ and $Q(0) = U_Q(0) \Lambda_Q V(0)^\top$ such that
$\frac{x_{i, \opt(i)}}{\|x_{i, \opt(i)}\|_2} = u^k_{\pi(i)} $ is the $\pi(i)$-th column of $U_K(0)$ and
$\frac{z_i}{\|z_i\|_2} = u^q_{\pi(i)} $ is the $\pi(i)$-th column of $U_Q(0)$ for all $i \in [n]$ and the rest of the columns of $U_K(0)$ and $U_Q(0)$ are the zero vectors. 
Then, only $\Lambda_K$ and $\Lambda_Q$ will be updated during the training, which leads to $K(t) = U_K(0) \Lambda_K(t) V(0)^\top$ and $Q(t) = U_Q(0) \Lambda_Q(t) V(0)^\top$ for all time $t$.
\end{proof}
\subsection{Proof of Theorem \ref{thm:generaltheorem}} \label{sec:proof_Thm2}
We recall Theorem \ref{thm:generaltheorem}:
\generaltheorem*

\begin{proof}
Under the conditions of  Lemma \ref{lem:alignment},  
it suffices to consider the dynamics of singular values $\Lambda_K$ and $\Lambda_Q$.
Similarly to our previous discussion, let
\begin{align}
& \xi_{k}=\operatorname{diag}(\Lambda_K) \in \mathbb{R}^{d}, \\
& \xi_{q}=\operatorname{diag}(\Lambda_Q) \in \mathbb{R}^{d}, \\
& \beta=\xi_{k} \circ \xi_{q} \in \mathbb{R}^{d}.
\end{align}
The gradient evaluated at $\xi_{k}$ and $\xi_{q}$ are
\begin{align}
& \dot{\xi}_k(t) = -\nabla_{\xi_k} \mathcal{L}\left(K, Q\right)=\frac{1}{n} \sum_{i=1}^{n} \exp \left(-\gamma_{i}^{\top} \S\left(g_{i}(t)\right)\right) \|\psi_i(t)\|_2 \|z_i\|_2e_{\pi(i)} \circ \xi_{q}(t), \\
& \dot{\xi}_q(t) = -\nabla_{\xi_q} \mathcal{L}\left(K, Q\right)=\frac{1}{n} \sum_{i=1}^{n} \exp \left(-\gamma_{i}^{\top} \S\left(g_{i}(t)\right)\right) \|\psi_i(t)\|_2 \|z_i\|_2e_{\pi(i)} \circ \xi_{k}(t), 
\end{align}
where  $\psi_i(t) = X_{i}^{\top} \Sigma\left(g_{i}(t)\right) \gamma_{i} \in \R^{d}$ and $e_i \in \R^d$ is a vector with $1$ at the $i$th component and $0$ elsewhere.

Consider the dynamics of the alternative reparameterization $\beta = \frac{w_+^{\odot 2}- w_-^{\odot 2}}{2}$ 
with the initialization $w_0 = w_+(0) = w_-(0)$ that satisfies the conditions of Lemma \ref{lem:equivalenceparameterization}. 
The corresponding dynamics of $w_+$ and $w_-$ are
\begin{align}
    \dot{w}_+(t) &= \frac{1}{n} \sum_{i=1}^{n} \exp \left(-\gamma_{i}^{\top} \S\left(g_{i}(t)\right)\right) \|\psi_i(t)\|_2 \|z_i\|_2e_{\pi(i)} \circ w_+(t), \\
    \dot{w}_-(t) &= -\frac{1}{n} \sum_{i=1}^{n} \exp \left(-\gamma_{i}^{\top} \S\left(g_{i}(t)\right)\right) \|\psi_i(t)\|_2 \|z_i\|_2e_{\pi(i)} \circ w_-(t).
\end{align}
Then, the dynamics of $\beta$ is
\begin{align}
    \beta(t) &= w_0^{\odot 2} \circ \sinh \left(\frac{2}{n} \int_{0}^{t} \sum_{i=1}^{n} \exp \left(-\gamma_{i}^{\top} \S\left(g_{i}(\tau)\right)\right)\|\psi_i(\tau)\|_2 \|z_i\|_2e_{\pi(i)} \diff \tau\right).
\end{align}

Since $K(t)$ and $Q(t)$ satisfy the alignment property for all $t$,
by \eqref{eq:psi_direction} and  Assumption \ref{item:(B1)} $(iv)$, it holds that $\frac{\psi_i(t)}{\|\psi_i(t)\|_2} = \frac{x_{i, \opt(i)}}{\|x_{i, \opt(i)}\|_2} = \pm \frac{z_i}{\|z_i\|_2}$ for all $i$ and $t$.
Then, we have
\begin{align}
    &\|\psi_i(t)\|_2 \|z_i\|_2  \\
    & = \sign({\la x_{i, \opt(i)}, z_i \ra)}\la \psi_i(t), z_i \ra \\
    & = \sign({\la x_{i, \opt(i)}, z_i \ra})\!\! \sum_{l \neq  \opt(i)}\! \S_{ \opt(i)}(g_{i}(t))\S_l\left(g_{i}(t)\right) (\gamma_{i, \opt(i)}-\gamma_{i,\nopt}) \la x_{i, \opt(i)}-x_{il}, z_i \ra, \label{eq:psi_z_norm}
\end{align}
where $\sign(\cdot)$ is a sign function.
Let $B_{il}  = \sign{(\la x_{i, \opt(i)},z_i\ra)} \la x_{i, \opt(i)}-x_{il}, z_i \ra e_{\pi(i)} \in \R^d$.
By  \eqref{eq:psi_z_norm}, the dynamics of $\beta$ can be written as 
\begin{align}
    \beta(t) &= w_0^{\odot 2} \circ \sinh \left(\frac{2}{n} \int_{0}^{t} \sum_{i=1}^{n} \sum_{l \neq  \opt(i)} \exp \left(-\gamma_{i}^{\top} \S\left(g_{i}( \tau)\right)\right)   \S(g_{i}( \tau))_{ \opt(i)}\S\left(g_{i}( \tau)\right)_l (\gamma_{i, \opt(i)}-\gamma_{i,\nopt}) B_{il} \diff \tau\right),
\end{align}
which has the same form as in \eqref{eq:beta_A}. By the analogous argument in the proof of Theorem \ref{thm:diagonaltheorem}, $\hat\beta^\infty = \lim_{t \rightarrow \infty}\frac{\beta(t)}{\min_{i,l \neq  \opt(i)}\beta(t)^\top B_{il}}$ is a global solution to
\begin{align} 
     \ \min _{\beta \in \mathbb{R}^{d}} \|\beta\|_{1} \quad \operatorname{ s.t. } \quad  \beta^{\top} B_{i l} \geq 1  \quad \forall \  l \neq \opt(i), \ i \in [n], \label{eq:BLSVM_A}
\end{align}
where $B_{il} = \sign{(\la x_{i, \opt(i)},z_i\ra)} \la x_{i, \opt(i)}-x_{il}, z_i \ra e_{\pi(i)} \in \R^d$.

Now, we show that if $\hat \beta^\infty$ is a global solution to \eqref{eq:BLSVM_A}, 
then $\hat W^\infty = \lim_{t \rightarrow \infty}\frac{U_K(0) \operatorname{Diag}(\beta(t)) U_Q(0)^\top}{\mu(t)}$ is a global solution to \eqref{eq:WNSVM}.
First,  for all $i,l \neq  \opt(i)$ and $t$, it holds that
\begin{align}
     \beta(t)^\top B_{il} &= \sign{(\la x_{i, \opt(i)},z_i\ra)} \la x_{i, \opt(i)}-x_{il}, z_i \ra \beta_{\pi(i)}(t) \\
    &= \left\la x_{i, \opt(i)}-x_{il}, \frac{x_{i, \opt(i)}}{\|x_{i, \opt(i)}\|_2}\right\ra \|z_i\|_2 \beta_{\pi(i)}(t) \\
    &= (x_{i, \opt(i)}-x_{il})^\top U_K(0) \operatorname{Diag}(\beta(t)) U_Q(0)^\top z_i \\
    &= (x_{i, \opt(i)}-x_{il})^\top W(t) z_i,
\end{align}
and 
\begin{align}
    U_K(0) \operatorname{Diag}(\hat \beta^\infty) U_Q(0)^\top &= \lim_{t \rightarrow \infty} \frac{U_K(0) \operatorname{Diag}(\beta(t)) U_Q(0)^\top}{\min_{i,l \neq  \opt(i)}\beta(t)^\top B_{il} }\\
&= \lim_{t \rightarrow \infty} \frac{U_K(0) \operatorname{Diag}(\beta(t)) U_Q(0)^\top}{\min_{i,l \neq  \opt(i)}(x_{i, \opt(i)}-x_{il})^\top W(t) z_i} \\
& = \hat W^\infty.
\end{align}
Hence, the constraints of \eqref{eq:BLSVM_A} and \eqref{eq:WNSVM} are equivalent with $\hat \beta^\infty$ and $\hat W^\infty$, respectively.

The stationary conditons of \eqref{eq:BLSVM_A} and \eqref{eq:WNSVM} are derived as follows:
\begin{align}
    \partial \|\beta\|_1 &\ni \sum_{i, l \neq  \opt(i)} \lambda_{il}B_{il} \\
    \partial \|W\|_* &\ni \sum_{i, l \neq  \opt(i)} \lambda_{il} (x_{i, \opt(i)}-x_{il})z_i^\top.
\end{align}
The subdifferential nuclear norm is defined as
\begin{align}
    \partial \|W\|_* = \cbr{UV^\top + E \ | \ P_U E = 0, \  P_V E = 0, \ \|E\|_{spec} \leq 1},
\end{align}
where $P_U$ and $P_V$ are the projection matrices onto the column spaces of $U$ and $V$, respectively and $\|E\|_{spec}$ is the spectral norm of $E$.

We first show that $\sum_{i,l\neq \opt(i)} \lambda_{il}(t) U_K(0) \operatorname{Diag}(B_{il}) U_Q(0)^\top$ $=\sum_{i, l \neq  \opt(i)} \lambda_{il}(t) (x_{i, \opt(i)}-x_{il})z_i^\top$
for all $t$.
Recall we let $\lambda_{il}(t) = \frac{2}{n\ln \mu(t)} \phi_{il}(t)$.
By Assumption \ref{item:(B1)} $(iii)$ and 
\eqref{eq:softmax_equal}, for each $x_{il}$, there exists unique $x_{il'}$ such that
$\S_l(g_i(t)) = \S_{l'}(g_i(t))$ for all $l \neq l'$ and $t$.  It follows from the definition of $\lambda_{il}(t)$ that $\lambda_{il}(t) = \lambda_{il'}(t)$ for all $t$.
Then, for such $x_{il}$ and $x_{il'}$, we have
\begin{align}
    &\lambda_{il}(t)(x_{i, \opt(i)}-x_{il}) + \lambda_{il'}(t)(x_{i, \opt(i)}-x_{il'}) \\
    & = \lambda_{il}(t)(2x_{i, \opt(i)}-x_{il} -x_{il'}) \\
    & = \lambda_{il}(t)(2x_{i, \opt(i)}-P_{x_{i, \opt(i)}}x_{il}-P_{x_{i, \opt(i)}}x_{il'}) \\
    & = \lambda_{il}(t)(x_{i, \opt(i)}-P_{x_{i, \opt(i)}}x_{il}) + \lambda_{il'}(t)(x_{i, \opt(i)}-P_{x_{i, \opt(i)}}x_{il'}) \\
    & = \lambda_{il}(t)\left\la x_{i, \opt(i)}-x_{il}, \frac{x_{i, \opt(i)}}{\|x_{i, \opt(i)}\|_2}\right\ra\frac{x_{i, \opt(i)}}{\|x_{i, \opt(i)}\|_2} + \lambda_{il'}(t)\left\la x_{i, \opt(i)}-x_{il'}, \frac{x_{i, \opt(i)}}{\|x_{i, \opt(i)}\|_2}\right\ra\frac{x_{i, \opt(i)}}{\|x_{i, \opt(i)}\|_2}. \label{eq:lambda_equivalence}
\end{align}

By   \eqref{eq:lambda_equivalence} and the definition of $B_{il}$, we have
\begin{align}
    \sum_{i,l\neq \opt(i)} \lambda_{il}(t) U_K(0) \operatorname{Diag}(B_{il}) U_Q(0)^\top & = \sum_{i,l\neq \opt(i)} \lambda_{il}(t) \left\la x_{i, \opt(i)}-x_{il}, \frac{x_{i, \opt(i)}}{\|x_{i, \opt(i)}\|_2}\right\ra\frac{x_{i, \opt(i)}}{\|x_{i, \opt(i)}\|_2}z_i^\top \\
    & = \sum_{i,l\neq \opt(i)} \lambda_{il}(t) (x_{i, \opt(i)}-x_{il})z_i^\top. \label{eq:stationary_equivalence1}
\end{align}

By Theorem 2 in \cite{Watson1992CharacterizationOT}, we have
\begin{align} \label{eq:stationary_equivalence2}
    \partial \|W\|_* = \text{conv}\cbr{U\Diag(d)V^\top \ |  \ W = U\Diag(\beta)V^\top \ d \in \partial\|\beta\|_1 },
\end{align}
where $\text{conv}\{\cdot\}$ denotes the convex hull of a set and $U,V$ are orthogonal matrices.
It follows from \eqref{eq:stationary_equivalence1}, \eqref{eq:stationary_equivalence2} and Theorem 2 in \cite{Watson1992CharacterizationOT}  that 
if $\hat \beta^\infty$ satisfies the stationary condition of \eqref{eq:BLSVM_A}, then $\hat W^\infty$ satisfies the stationary condition of \eqref{eq:WNSVM}.

\end{proof}

\end{appendices}

\end{document}